\documentclass[final,12pt]{arxiv}

\usepackage[utf8]{inputenc} 
\usepackage[T1]{fontenc}    
\usepackage{hyperref}       
\usepackage{url}            
\usepackage{booktabs}       
\usepackage{amsfonts}       
\usepackage{nicefrac}       
\usepackage{microtype}      
\usepackage{xcolor}         

\usepackage{dsfont}
\usepackage{enumitem}
\usepackage[toc,page,header]{appendix}
\usepackage{lipsum,wrapfig}

\usepackage[english]{babel}
\usepackage{amsmath, amssymb}
\usepackage{graphicx}
\usepackage{latexsym}
\usepackage{graphicx}
\usepackage{cases}
\usepackage[mathscr]{euscript}
\usepackage{xcolor}
\usepackage{colortbl}
\usepackage{thmtools}
\usepackage{thm-restate}

\usepackage{mathtools}
\usepackage{subcaption}

\usepackage{multirow}
\usepackage{wrapfig}

\usepackage{pifont}

\newtheorem{lemma}[theorem]{Lemma}

\usepackage{bm}

\let\vec\undefined
\usepackage{MnSymbol}
\DeclareMathAlphabet\mathbb{U}{msb}{m}{n}
\usepackage{xpatch}

\def\Rset{\mathbb{R}}

\let\P\undefined

\DeclareMathOperator*{\P}{\mathbb{P}}
\DeclareMathOperator*{\E}{\mathbb E}

\DeclarePairedDelimiter{\abs}{\lvert}{\rvert} 
\DeclarePairedDelimiter{\bracket}{[}{]}
\DeclarePairedDelimiter{\curl}{\{}{\}}
\DeclarePairedDelimiter{\norm}{\lVert}{\rVert}
\DeclarePairedDelimiter{\paren}{(}{)}

\newcommand{\sC}{{\mathscr C}}
\newcommand{\sD}{{\mathscr D}}
\newcommand{\sE}{{\mathscr E}}

\newcommand{\sH}{{\mathscr H}}

\newcommand{\sM}{{\mathscr M}}

\newcommand{\sR}{{\mathscr R}}

\newcommand{\sX}{{\mathscr X}}
\newcommand{\sY}{{\mathscr Y}}

\newcommand{\sfL}{{\mathsf L}}

\newcommand{\Rad}{\mathfrak R}

\newcommand{\h}{\widehat}
\newcommand{\ov}{\overline}
\newcommand{\wt}{\widetilde}
\newcommand{\e}{\epsilon}

\newcommand{\ignore}[1]{}

\def\Nset{\mathbb{N}}

\usepackage[toc, page, header]{appendix}
\usepackage{multirow}

\hypersetup{
  breaklinks   = true, 
  colorlinks   = true, 
  urlcolor     = blue, 
  linkcolor    = blue, 
  citecolor   = blue 
}
\usepackage{times}

\title[\texorpdfstring{$\sH$}{H}-Consistency Guarantees for Regression]
      {\texorpdfstring{$\sH$}{H}-Consistency Guarantees for Regression}

\arxivauthor{%
 \Name{Anqi Mao} \Email{aqmao@cims.nyu.edu}\\
 \addr Courant Institute of Mathematical Sciences, New York%
 \AND
 \Name{Mehryar Mohri} \Email{mohri@google.com}\\
 \addr Google Research and Courant Institute of Mathematical Sciences, New York%
 \AND
 \Name{Yutao Zhong} \Email{yutao@cims.nyu.edu}\\
 \addr Courant Institute of Mathematical Sciences, New York%
}

\begin{document}

\maketitle

\begin{abstract}
  We present a detailed study of $\sH$-consistency bounds for
  regression. We first present new theorems that generalize the tools
  previously given to establish $\sH$-consistency bounds. This
  generalization proves essential for analyzing $\sH$-consistency
  bounds specific to regression. Next, we prove a series of novel
  $\sH$-consistency bounds for surrogate loss functions of the squared
  loss, under the assumption of a symmetric distribution and a bounded
  hypothesis set.  This includes positive results for the Huber loss,
  all $\ell_p$ losses, $p \geq 1$, the squared $\e$-insensitive loss,
  as well as a negative result for the $\e$-insensitive loss used in
  squared Support Vector Regression (SVR).  We further leverage our
  analysis of $\sH$-consistency for regression and derive principled
  surrogate losses for adversarial regression
  (Section~\ref{sec:adversarial}). This readily establishes novel
  algorithms for adversarial regression, for which we report favorable
  experimental results in Section~\ref{sec:experiments}.

\end{abstract}



\section{Introduction}

Learning algorithms often optimize loss functions that differ from the
originally specified task. In classification, this divergence
typically arises due to the computational intractability of optimizing
the original loss or because it lacks certain desirable properties
like differentiability or smoothness. In regression, the shift may
occur because the surrogate loss used exhibits more favorable
characteristics, such as handling outliers or ensuring sparser
solutions. For instance, the Huber loss and $\ell_1$ loss are used to
mitigate the impact of outliers since the squared loss is known to be
sensitive to the presence of outliers, while $\epsilon$-insensitive
losses promote sparsity. But, what guarantees do we have when training
with a loss function distinct from the target squared loss?

Addressing this question can have significant implications in the
design of regression algorithms. It can also strongly benefit the
design of useful surrogate losses for other related problems, such as
adversarial regression, as we shall see.

\ignore{
Moreover, as we will show later in the paper, not just training with
the $\ell_2$ target loss also becomes useful for adversarial
regression.  Furthermore, this study can also be useful for later
studies of surrogate losses for other target losses in
regression. Specifically, we will show that under certain conditions,
loss functions such as the Huber loss and the squared $\e$-insensitive
loss can be \emph{$\sH$-consistent.}
}

The statistical properties of surrogate losses have been extensively
studied in the past. In particular, the Bayes-consistency of various
convex loss functions, including margin-based surrogate losses in
binary classification \citep*{Zhang2003, bartlett2006convexity}, and
other loss function families for multi-classification
\cite{zhang2004statistical, tewari2007consistency,
  steinwart2007compare}, has been examined in detail.

However, prior work by \citet{long2013consistency} has highlighted the
limitations of Bayes-consistency, since it does not account for the
hypothesis set adopted. They established that for some hypothesis sets
and distributions, algorithms minimizing Bayes-consistent losses may
retain a constant expected error, while others minimizing inconsistent
losses tend to have an expected error approaching zero. This indicates
the significant role of the chosen hypothesis set in consistency.

Recent seminal work by \citet*{awasthi2022h,awasthi2022multi} and
\citet*{mao2023cross,MaoMohriZhong2023ranking,
  MaoMohriZhong2023structured,MaoMohriZhong2023characterization} has
analyzed \emph{$\sH$-consistency bounds} for broad families of
surrogate losses in binary classication, multi-class classification,
structured prediction, and abstention
\citep{MaoMohriMohriZhong2023twostage}. These bounds are more
informative than Bayes-consistency since they are hypothesis
set-specific and do not require the entire family of measurable
functions.  Moreover, they offer finite sample, non-asymptotic
guarantees.
In light of these recent guarantees, the following questions naturally
arise: Can we derive a non-asymptotic analysis of regression taking
into account the hypothesis set? How can we benefit from that
analysis? 

While there is some previous work exploring Bayes-consistency in
regression \citep{Caponnetto2005,ChristmannSteinwart2007,
  steinwart2007compare}, we are not aware of any prior
$\sH$-consistency bounds or similar finite sample guarantees for
surrogate losses in regression, such as, for example, the Huber loss
or the squared $\e$-insensitive loss.

This paper presents the first in-depth study of $\sH$-consistency
bounds in the context of regression. We first present new theorems
that generalize the tools previously given by
\citet{awasthi2022h,awasthi2022multi} and
\citet{mao2023cross,MaoMohriZhong2023ranking,
  MaoMohriZhong2023structured,MaoMohriZhong2023characterization} to
establish $\sH$-consistency bounds
(Section~\ref{sec:general-theorems}). This generalization proves
essential in regression for analyzing $\sH$-consistency bounds for
surrogate losses such as Huber loss and the squared $\e$-insensitive
loss. It also provides finer bounds for the $\ell_1$ loss.

Next, we prove a series of $\sH$-consistency bounds for
surrogate loss functions of the squared loss, under the assumption of
a symmetric distribution and a bounded hypothesis set
(Section~\ref{sec:hcons}).
We prove the first $\sH$-consistency bound for the Huber loss, which
is a commonly used surrogate loss used to handle outliers, contingent
upon a specific condition concerning the Huber loss parameter $\delta$
and the distribution mass around the mean.  We further prove that this
condition is necessary when $\sH$ is realizable.

We then extend our analysis to cover $\sH$-consistency bounds for
$\ell_p$ losses, for all values of $p \geq 1$. In particular,
remarkably, we give guarantees for the $\ell_1$ loss and $\ell_p$
losses with $p \in (1, 2)$.
We further analyze the $\e$-insensitive and the squared $\e$-insensitive
losses integral to the definition of the SVR (Support Vector Regression)
and quadratic SVR algorithms \citep{Vapnik2000}. These loss functions
and SVR algorithms admit the benefit of yielding sparser solutions.
We give the first $\sH$-consistency bound for the quadratic
$\e$-insensitive loss. We also prove a negative result for the
$\e$-insensitive loss: this loss function used in the definition
of SVR does not admit $\sH$-consistency bounds with respect to
the squared loss, even under some additional assumptions on the
parameter $\e$ and the distribution.

Subsequently, leveraging our analysis of $\sH$-consistency for
regression, we derive principled surrogate losses for adversarial
regression (Section~\ref{sec:adversarial}). This readily establishes a
novel algorithm for adversarial regression, for which we report
favorable experimental results in Section~\ref{sec:experiments}.

\textbf{Previous work.} Bayes-consistency has been extensively studied
in various learning problems. These include binary classification
\citep{Zhang2003, bartlett2006convexity}, multi-class classification
\citep{zhang2004statistical, tewari2007consistency,
  narasimhan2015consistent, finocchiaro2019embedding, wang2020weston,
  frongillo2021surrogate, wang2023classification}, ranking
\citep{menon2014bayes, gao2015consistency, uematsu2017theoretically},
multi-label classification \citep{gao2011consistency,
  koyejo2015consistent, zhang2020convex}, structured prediction
\citep{ciliberto2016consistent, osokin2017structured,
  blondel2019structured}, and ordinal regression
\citep{pedregosa2017consistency}. The concept of $\sH$-consistency has
been studied under the realizable assumption in
\citep{long2013consistency,zhang2020bayes}. The notion of
$\sH$-consistency bounds in classification is due to
\citet{awasthi2022h, awasthi2022multi}. $\sH$-consistency bounds have
been further analyzed in scenarios such as multi-class classification
\citep{mao2023cross,zheng2023revisiting,MaoMohriZhong2023characterization},
ranking \citep{MaoMohriZhong2023ranking,MaoMohriZhong2023rankingabs}, structured prediction
\citep{MaoMohriZhong2023structured}, and abstention
\citep{MaoMohriZhong2024deferral,MaoMohriZhong2024score,MaoMohriZhong2024predictor,MohriAndorChoiCollinsMaoZhong2024learning}.

However, in the context of regression, there is limited work on the
consistency properties of surrogate losses. The main related work we
are aware are
\citep{Caponnetto2005,ChristmannSteinwart2007,steinwart2007compare}. In
particular, \citet{steinwart2007compare} studied Bayes-consistency for
a family of regression surrogate losses including $\ell_p$, but
without presenting any non-asymptotic bound. Nevertheless, we partly
benefit from this previous work. In particular, we adopt the same
symmetric and bounded distribution assumption.

\ignore{
We fill a gap in the literature by studying $\sH$-consistency bounds
in the context of regression.  This requires the generalization of
existing tools previously used for proving $\sH$-consistency bounds in
the classification context \citep{awasthi2022h,awasthi2022multi}. Such
a generalization is not only necessary for deriving bounds in the
regression context but also has the potential to provide better bounds
in some cases compared to the existing tools.

We leverage our general tools to prove a series of $\sH$-consistency
bounds for common surrogate losses of the squared loss. They are both
non-asymptotic and specific to the hypothesis set. These results
either represent a significant generalization of the Bayes-consistency
results shown in \citep{steinwart2007compare}, or they yield novel
Bayes-consistency results as a by-product for a new family of
surrogate losses that were not previously studied in
\citep{steinwart2007compare}. For example, Corollary~\ref{cor:Lp}
extends Bayes-consistency for the $\ell_p$ loss to $\sH$-consistency
for any bounded and realizable hypothesis sets $\sH$;
Corollary~\ref{cor:e-sen-sq} implies the Bayes-consistency of the
squared $\e$-insensitive loss, which has not been previously studied
in \citep{steinwart2007compare}
}

\section{Preliminaries}
\label{sec:preliminaries}

\textbf{Bounded regression.} We first introduce the learning scenario
of bounded regression. We denote by $\sX$ the input space, $\sY$ a
measurable subset of $\Rset$, and $\sD$ a distribution over $\sX\times
\sY$. As for other supervised learning problems, the learner receives
a labeled sample $S = \paren*{(x_1, y_1), \ldots, (x_m, y_m)}$ drawn
i.i.d.\ according to $\sD$.

The measure of error is based on the magnitude of the difference
between the predicted real-valued label and the true label.
The function used to measure the error is denoted as $\sfL\colon \sY
\times \sY \to \Rset_{+}$. Let $L \colon (h, x, y) \mapsto \sfL(h(x),
y)$ be the associated loss function. Some common examples of loss
functions used in regression are the squared loss $\ell_2$, defined by
$\sfL(y', y) = \abs*{y' - y}^2$ for all $y, y'\in \sY$; or more
generally the $\ell_p$ loss defined by $\sfL(y', y) = \abs*{y' -
  y}^p$, for $p \geq 1$. The squared loss is known to be quite
sensitive to outliers. An alternative more robust surrogate loss is
the Huber loss $\ell_{\delta}$ \citep{Huber1964}, which is defined for
a parameter $\delta > 0$ as the following combination of the $\ell_2$
and $\ell_1$ loss functions: $\sfL(y', y) = \frac12(y' - y)^2$ if
$\abs*{y' - y}\leq \delta$, $\paren*{\delta\abs*{y' - y} - \frac{1}{2}
\delta^2}$ otherwise. The $\e$-insensitive loss $\ell_{\e}$ and the
squared $\e$-insensitive loss $\ell_{\rm{sq}-\e}$ \citep{Vapnik2000}
are defined by $\sfL(y', y) = \max\curl*{|y' - y| - \e, 0}$ and
$\sfL(y', y) = \max\curl*{|y' - y|^2 - \e^2, 0}$, for
some $\e > 0$.

\noindent \textbf{Bayes-Consistency.} Given a loss function $L$, we
denote by $\sE_{L}(h)$ the generalization error of a hypothesis $h \in
\sH$, and by $\sE^*_{L}\paren*{\sH}$ the best-in-class error for a
hypothesis set $\sH$:
\begin{equation*}
  \sE_{L}(h) = \E_{(x, y) \sim \sD}\bracket*{L(h, x, y)} \quad
  \sE^*_{L}\paren*{\sH} = \inf_{h \in \sH} \sE_{L}(h).
\end{equation*}
A desirable property of surrogate losses in regression is
\emph{Bayes-consistency}
\citep{Zhang2003,bartlett2006convexity,steinwart2007compare}, that is,
minimizing the surrogate losses $L$ over the family of all measurable
functions $\sH_{\rm{all}}$ leads to the minimization of the squared
loss $\ell_2$ over $\sH_{\rm{all}}$.  We say that $L$ is
\emph{Bayes-consistent} with respect to $\ell_2$, if, for all
distributions and sequences of $\{h_n\}_{n \in \Nset} \subset
\sH_{\rm{all}}$, $\lim_{n \to +\infty}
\sE_{L}(h_n)-\sE^*_{L}(\sH_{\rm{all}}) = 0$ implies $\lim_{n \to
  +\infty} \sE_{\ell_2}(h_n)-\sE^*_{\ell_2}(\sH_{\rm{all}}) = 0$.
Bayes-consistency stands as an essential prerequisite for a surrogate
loss. Nonetheless, it has some shortcomings: it is only an asymptotic
property and it fails to account for the hypothesis set $\sH$
\citep{awasthi2022h,awasthi2022multi}.

\noindent \textbf{$\sH$-Consistency bounds.}  In contrast with
Bayes-consistency, $\sH$-Consistency bounds take into account the
specific hypothesis set $\sH$ and are non-asymptotic. Given a
hypothesis set $\sH$, we say that a regression loss function $L$
admits an \emph{$\sH$-consistency bound with respect to $\ell_2$}
\citep{awasthi2022h,awasthi2022multi}, if for some non-decreasing
function $f\colon \Rset_{+}\to \Rset_{+}$, for all distributions and
all $h \in \sH$, the following inequality holds:
\begin{equation*}
  \sE_{\ell_2}(h) - \sE^*_{\ell_2}(\sH)
  \leq f \paren*{\sE_{L}(h) - \sE^*_{L}(\sH)}.    
\end{equation*}
Thus, when the $L$-estimation error can be reduced to some $\eta > 0$,
the squared loss estimation error is upper bounded by $f(\eta)$. An
$\sH$-Consistency bound is a stronger and more informative property
than Bayes-consistency, which is implied by taking the limit.

In the next section, we will prove $\sH$-consistency bounds for
several common surrogate regression losses with respect to the squared
loss $\ell_2$. A by-product of these guarantees is the
Bayes-consistency of these losses.

For a regression loss function $L$ and a hypothesis $h$, the
generalization error can be expressed as follows:
\begin{equation*}
\sE_{L}(h) 
= \E_{x}\bracket*{\E_y \bracket*{\sfL(h(x), y) \mid x}}
= \E_{x}\bracket*{\sC_{L}(h, x)},
\end{equation*}
where $\sC_{L}(h, x)$ is the \emph{conditional error} $\E_y
\bracket*{\sfL(h(x), y) \mid x}$. We also write $\sC^*_{L}(\sH, x)$ to
denote the best-in-class conditional error defined by $\sC^*_{L}(\sH,
x) = \inf_{h \in \sH} \sC_{L}(h, x)$. The conditional regret or
calibration gap, $\Delta \sC_{L, \sH}(h, x)$, measures the difference
between the conditional error of $h$ and the best-in-class conditional
error: $\Delta \sC_{L, \sH}(h, x) = \sC_{L}(h, x) - \sC^*_{L}(\sH, x)$.  A
generalization of conditional regret is the conditional $\e$-regret,
defined as: $\bracket*{\Delta \sC_{L, \sH}(h, x)}_{\e} = \Delta \sC_{L,
  \sH}(h, x) 1_{\Delta \sC_{L, \sH}(h, x) > \e}$.

A key term appearing in our bounds is the minimizability gap, defined
for a loss function $L$ and a hypothesis set $\sH$ as $\sM_{L}(\sH) =
\sE^*_{L}(\sH) - \E_{x}\bracket*{\sC^*_{L}(\sH, x)}$. It quantifies
the discrepancy between the best-in-class generalization error and the
expected best-in-class conditional error.  An alternative expression
for the minimizability gap is: $\sM_{L}(\sH) = \inf_{h \in \sH}
\E_{x}\bracket*{\sC_{L}(\sH, x)} - \E_{x}\bracket*{\inf_{h \in \sH}
  \sC_{L}(\sH, x)}$. Due to the super-additivity of the infimum, the
minimizability gap is always non-negative.  As shown by \citet[Lemma
  2.5, Theorem~3.2]{steinwart2007compare}, for the family of all
measurable functions, the equality $\sE^*_{L}(\sH_{\rm{all}}) =
\E_{x}\bracket*{\sC^*_{L}(\sH_{\rm{all}}, x)}$ holds. Thus, the
minimizability gap can be bounded above by the approximation error
$\sE^*_{L}(\sH) - \sE^*_{L}(\sH_{\rm{all}})$. The minimizability gap
becomes zero when when $\sH = \sH_{\rm{all}}$ or, more broadly, when
$\sE^*_{L}(\sH) = \sE^*_{L}\paren*{\sH_{\rm{all}}}$.

\section{General $\sH$-consistency theorems}
\label{sec:general-theorems}

To derive $\sH$-consistency bounds for regression, we first give
two key theorems establishing that if a convex or concave function
provides an inequality between the conditional regrets of regression
loss functions $L_1$ and $L_2$, then this inequality translates into
an $\sH$-consistency bound involving the minimizability gaps of $L_1$ and
$L_2$.

\begin{restatable}[General $\sH$-consistency bound -- convex function]
  {theorem}{HConsBoundPsi}
  \label{theorem:H-ConsBoundPsi}
  Let $\sD$ denote a distribution over $\sX \times \sY$. Assume that
  there exists a convex function $\Psi\colon \Rset_+ \to \Rset$
  with $\Psi(0) \geq 0$, a positive function $\alpha\colon \sH \times
  \sX \to \Rset^*_+$ with $\sup_{x \in \sX} \alpha(h, x) < +\infty$ for all $h
  \in \sH$, and $\e \geq 0$ such that the following holds for all
  $h\in \sH$, $x\in \sX$:
  $\Psi\paren*{\bracket*{\Delta\sC_{L_2,\sH}(h, x)}_\e} \leq
  \alpha(h, x) \, \Delta \sC_{L_1,\sH}(h, x)$. Then, for any
  hypothesis $h \in \sH$, the following inequality holds:
  \ifdim\columnwidth=\textwidth
{
\begin{align*}
      \Psi\paren*{\sE_{L_2}(h) - \sE_{L_2}^*(\sH) + \sM_{L_2}(\sH)} \leq \bracket[\Big]{\sup_{x \in \sX} \alpha(h, x)} \paren*{\sE_{L_1}(h) - \sE_{L_1}^*(\sH) + \sM_{L_1}(\sH)} + \max \curl*{\Psi(0), \Psi(\e)}.
    \end{align*}
}
\else{
    \begin{align*}
      & \Psi\paren*{\sE_{L_2}(h) - \sE_{L_2}^*(\sH) + \sM_{L_2}(\sH)}\\
      & \quad \leq \bracket[\Big]{\sup_{x \in \sX} \alpha(h, x)} \paren*{\sE_{L_1}(h) - \sE_{L_1}^*(\sH) + \sM_{L_1}(\sH)}\\
      & \qquad + \max \curl*{\Psi(0), \Psi(\e)}.
    \end{align*}
}
\fi
\end{restatable}

\begin{restatable}[General $\sH$-consistency bound -- concave function]
  {theorem}{HConsBoundGamma}
\label{theorem:H-ConsBoundGamma}
Let $\sD$ denote a distribution over $\sX \times \sY$.  Assume that
there exists a concave function $\Gamma\colon \Rset_+\to \Rset$, a
positive function $\alpha\colon \sH \times \sX \to \Rset^*_+$ with
$\sup_{x \in \sX} \alpha(h, x) < +\infty$ for all $h \in \sH$, and
$\epsilon\geq0$ such that the following holds for all $h \in \sH$,
$x\in \sX$: $\bracket*{\Delta\sC_{L_2,\sH}(h, x)}_\e \leq \Gamma
\paren*{\alpha(h, x) \Delta\sC_{L_1,\sH}(h, x)}$.  Then, for any
hypothesis $h \in \sH$, the following inequality holds
\ifdim\columnwidth=\textwidth
{
\begin{align*}
  \sE_{L_2}(h) - \sE_{L_2}^*(\sH) + \sM_{L_2}(\sH) \leq \Gamma\paren*{\bracket[\Big]{\sup_{x \in \sX} \alpha(h, x)} \paren*{\sE_{L_1}(h) - \sE_{L_1}^*(\sH) + \sM_{L_1}(\sH)}} + \e.
\end{align*}
}
\else{
\begin{align*}
  & \sE_{L_2}(h) - \sE_{L_2}^*(\sH) + \sM_{L_2}(\sH)\\
  &  \leq \Gamma\paren*{\bracket[\Big]{\sup_{x \in \sX} \alpha(h, x)} \paren*{\sE_{L_1}(h) - \sE_{L_1}^*(\sH) + \sM_{L_1}(\sH)}} + \e.
\end{align*}
}
\fi
In the special case where $\Gamma(x) = x^{\frac{1}{q}}$ for some $q
\geq 1$ with conjugate $p \geq 1$, that is $\frac{1}{p} + \frac{1}{q}
= 1$, for any $h \in \sH$, the following inequality holds, assuming
$\E_X \bracket[\big]{\alpha^{\frac{p}{q}}(h, x)}^{\frac{1}{p}} <
+\infty$ for all $h \in \sH$:
\ifdim\columnwidth=\textwidth
{
\begin{align*}
\sE_{L_2}(h) - \sE_{L_2}^*(\sH) + \sM_{L_2}(\sH) \leq  \E_X \bracket*{\alpha^{\frac{p}{q}}(h, x)}^{\frac{1}{p}}
  \E_X \bracket*{\sE_{L_1}(h) - \sE_{L_1}^*(\sH)
    + \sM_{L_1}(\sH)}^{\frac{1}{q}} + \e.
\end{align*}
}
\else{
\begin{align*}
& \sE_{L_2}(h) - \sE_{L_2}^*(\sH) + \sM_{L_2}(\sH)\\
  &  \leq  \E_X \bracket*{\alpha^{\frac{p}{q}}(h, x)}^{\frac{1}{p}}
  \E_X \bracket*{\sE_{L_1}(h) - \sE_{L_1}^*(\sH)
    + \sM_{L_1}(\sH)}^{\frac{1}{q}} + \e.
\end{align*}
}
\fi
\end{restatable}
Theorems~\ref{theorem:H-ConsBoundPsi}
and~\ref{theorem:H-ConsBoundGamma} provide significantly more general
tools for establishing $\sH$-consistency bounds than previous results
from \citep[Theorems~1 and 2]{awasthi2022h} and \citep[Theorems~1 and
  2]{awasthi2022multi} for binary and multi-class classification. They
offer a more general framework for establishing consistency bounds by
allowing for non-constant functions $\alpha$.  This generalization is
crucial for analyzing consistency bounds in regression, where $\alpha$
may not be constant for certain surrogate losses (e.g., Huber loss,
squared $\e$-insensitive loss). Our generalized theorems also enable
finer consistency bounds, as demonstrated later in the case of the
$\ell_1$ loss. The proofs of Theorems~\ref{theorem:H-ConsBoundPsi}
and~\ref{theorem:H-ConsBoundGamma} are included in
Appendix~\ref{app:general}.

To leverage these general theorems, we will characterize the
best-in-class conditional error and the conditional regret of the
squared loss. We first introduce some definitions we will need.  We
say that the conditional distribution is \emph{bounded by $B > 0$} if,
for all $x \in \sX, \P\paren*{|Y| \leq B \mid X = x} = 1$. We
say that a hypothesis set $\sH$ is bounded by $B > 0$ if, $|h(x)| \leq
B$ for all $h \in \sH$ and $x \in \sX$, and all values in $[\minus B,
  \plus B]$ are attainable by $h(x)$, $h \in \sH$. The
conditional mean of the distribution at $x$ is denoted as: $\mu(x) =
\E[y \mid x]$.

\begin{restatable}{theorem}{SquaredLossBound}
\label{thm:squared-bound}
Assume that the conditional distribution and the hypothesis set $\sH$
are bounded by $B > 0$. Then, the best-in-class conditional error and
the conditional regret of the squared loss can be characterized as:
for all $ h \in \sH, x \in \sX$,
\begin{align*}
  \sC^*_{\ell_2}(\sH, x) &= \sC_{\ell_2}(\mu(x), x)
  = \E[y^2 \mid x] - \paren*{\mu(x)}^2 \\
\Delta \sC_{\ell_2, \sH}(h, x) & = \paren*{h(x) - \mu(x) }^2.
\end{align*}
\end{restatable}
Refer to Appendix~\ref{app:general} for the proof.
As in \citep{steinwart2007compare}, for our analysis, we will focus
specifically on symmetric distributions, where the conditional mean
and the conditional median coincide. This is because, otherwise, as shown
by \citet[Proposition~4.14]{steinwart2007compare} the squared loss is
essentially the only distance-based and locally Lipschitz continuous
loss function that is Bayes-consistent with respect to itself for all
bounded conditional distributions.

A distribution $\sD$ over $\sX \times \sY$ is said to be
\emph{symmetric} if and only if for all $x \in \sX$, there exits $y_0
\in \Rset$ such that $\sD_{y | x} (y_0 - A) = \sD_{y | x} (y_0 + A)$
for all measurable $A \subset [0, \plus \infty)$.  The next result
  characterizes the best-in-class predictor for any symmetric
  regression loss functions for such distributions.
  
\begin{restatable}{theorem}{General}
\label{thm:general}
Let $\psi \colon \Rset \to \Rset$ be a symmetric function such that
$\psi(x) = \psi(-x)$ for all $x \in \Rset$. Furthermore, $\psi(x) \geq
0$ for all $x$ in its domain and it holds that $\psi(0) = 0$. Assume
that the conditional distribution and the hypothesis set $\sH$ is
bounded by $B > 0$. Assume that the distribution is symmetric and the
regression loss function is given by $\sfL(y', y) = \psi(y' -
y)$. Then, we have $\sC^*_{L}(\sH, x) = \sC_{L}(\mu(x), x) $.
\end{restatable}
The proof is included in Appendix~\ref{app:general}. It is
straightforward to see that all the previously mentioned regression
loss functions satisfy the assumptions in
Theorem~\ref{thm:general}. Therefore, for these loss functions, the
best-in-class conditional error is directly characterized by
Theorem~\ref{thm:general}. Furthermore, if we have $x \mapsto \mu(x)
\in \sH$, then under the same assumption, we have $\sE^*_{L}(\sH) =
\E_{x}\bracket*{\sC^*_{L}(\sH, x)} = \E_{x} \bracket*{\sC_{L}(\mu(x),
  x)}$ and thus the minimizability gap vanishes: $\sM_{L}(\sH) = 0$.
\begin{definition}
  A hypothesis set $\sH$ is said to be \emph{realizable} if the
  function that maps $x$ to the conditional mean $\mu(x)$ is included
  in $\sH$: $x \mapsto \mu(x) \in \sH$.
\end{definition}

\begin{corollary}
\label{cor:general}
Under the same assumption as in Theorem~\ref{thm:general}, for
realizable hypothesis sets, we have $\sM_{L}(\sH) = 0$.
\end{corollary}

\section{$\sH$-Consistency bounds for regression}
\label{sec:hcons}

In this section, we will analyze the $\sH$-consistency of several
regression loss functions with respect to the squared loss.

\subsection{Huber Loss}
\label{sec:huber}

The Huber loss $\ell_{\delta} \colon (h, x, y) \mapsto \frac12(h(x) -
y)^2 1_{\abs*{h(x) - y}\leq \delta} + \paren*{\delta\abs*{h(x) - y} -
  \frac12 \delta^2} 1_{\abs*{h(x) - y} > \delta}$ is a frequently used
loss function in regression for dealing with outliers. It imposes
quadratic penalties on small errors and linear penalties on larger
ones. The next result provides $\sH$-consistency bounds for the Huber
loss with respect to the squared loss.

\begin{restatable}{theorem}{Huber}
\label{thm:huber}
Assume that the distribution is symmetric, the conditional
distribution and the hypothesis set $\sH$ are bounded by $B > 0$.
Assume that $p_{\min}(\delta) = \inf_{x \in \sX} \P(0 \leq \mu(x) -
y \leq \delta \mid x)$ is positive.  Then, for all $h \in \sH$, the
following $\sH$-consistency bound holds:
\ifdim\columnwidth=\textwidth
{
\begin{align*}
\sE_{\ell_2}(h) - \sE^*_{\ell_2}(\sH) + \sM_{\ell_2}(\sH) \leq \frac{\max \curl*{\frac{2B}{\delta}, 2}}
  {p_{\min}(\delta)}
  \paren*{\sE_{\ell_{\delta}}(h) - \sE^*_{\ell_{\delta}}(\sH) + \sM_{\ell_{\delta}}(\sH)}.
\end{align*}
}
\else{
\begin{multline*}
  \sE_{\ell_2}(h) - \sE^*_{\ell_2}(\sH) + \sM_{\ell_2}(\sH) \leq \\
  \frac{\max \curl*{\frac{2B}{\delta}, 2}}
  {p_{\min}(\delta)}
  \paren*{\sE_{\ell_{\delta}}(h) - \sE^*_{\ell_{\delta}}(\sH) + \sM_{\ell_{\delta}}(\sH)}.
\end{multline*}
}
\fi
\end{restatable}
The proof is presented in Appendix~\ref{app:huber}. It leverages the
general Theorem~\ref{theorem:H-ConsBoundPsi} with $\alpha(h, x) =
\P(0 \leq \mu(x) - y \leq \delta \mid x)$. Note that the
previous established general tools for $\sH$-consistency bounds
\citep{awasthi2022h,awasthi2022multi} require $\alpha$ to be constant,
which is not applicable in this context. This underscores the
necessity of generalizing previous tools to accommodate any positive
function $\alpha$.

As shown by Corollary~\ref{cor:general}, when $\sH$ is realizable, the
minimizability gap vanishes. Thus, by Theorem~\ref{thm:huber}, we
obtain the following corollary.
\begin{corollary}
\label{cor:huber}
Assume that the distribution is symmetric, the conditional
distribution is bounded by $B > 0$, and the hypothesis set $\sH$ is
realizable and bounded by $B > 0$. Assume that $p_{\min}(\delta) = \inf_{x \in
  \sX} \P(0 \leq \mu(x) - y \leq \delta \mid x)$ is positive.
Then, for all $h \in \sH$, the following $\sH$-consistency bound
holds:
\ifdim\columnwidth=\textwidth
{
\begin{align*}
\sE_{\ell_2}(h) - \sE^*_{\ell_2}(\sH)
  \leq \frac{\max \curl*{\frac{2B}{\delta}, 2}}{p_{\min}(\delta)}
  \paren*{\sE_{\ell_{\delta}}(h) - \sE^*_{\ell_{\delta}}(\sH)}.
\end{align*}
}
\else{
\begin{align*}
& \sE_{\ell_2}(h) - \sE^*_{\ell_2}(\sH)\\
  & \quad
  \leq \frac{\max \curl*{\frac{2B}{\delta}, 2}}{p_{\min}(\delta)}
  \paren*{\sE_{\ell_{\delta}}(h) - \sE^*_{\ell_{\delta}}(\sH)}.
\end{align*}
}
\fi
\end{corollary}
Corollary~\ref{cor:huber} implies the Bayes-consistency of the Huber
loss when $p_{\min}(\delta) > 0$, by taking the limit on both sides of
the bound.  Note that, as the value of $\delta$ increases,
$\frac{2B}{\delta}$ decreases and $p_{\min}(\delta)$ increases, which
improves the linear dependency on the Huber loss estimation error in
this bound.  However, this comes at the price of an Huber loss more
similar to the squared loss and thus a higher sensitivity to outliers.
Thus, selecting an appropriate value for $\delta$ involves considering
these trade-offs.

The bound is uninformative when the probability mass
$p_{\min}(\delta)$ is zero. However, the following theorem shows that
the condition $p_{\min}(\delta) > 0$ is necessary and that otherwise,
in general, the Huber loss is not $\sH$-consistent with respect to the
squared loss.

\begin{restatable}{theorem}{HuberN}
\label{theorem:huber-n}
Assume that the distribution is symmetric, the conditional
distribution is bounded by $B > 0$, and the hypothesis set $\sH$ is
realizable and bounded by $B > 0$. Then, the Huber loss
$\ell_{\delta}$ is not $\sH$-consistent with respect to the squared
loss.
\end{restatable}
Refer to Appendix~\ref{app:huber} for the proof, which consists of
considering a distribution that concentrates on an input $x$ with
$\P(Y = y \mid x) = \frac{1}{2} = \P(Y = 2\mu(x) - y \mid x)$, where
$-B \leq y < \mu(x) \leq B$ and $\mu(x) - y > \delta$. Then, we show
that both $\ov h \colon x \mapsto y + \delta$ and $h^* \colon x
\mapsto \mu(x)$ are best-in-class predictors of the Huber loss, while
the best-in-class-predictor of the squared loss is uniquely $h^*
\colon x \mapsto \mu(x)$.

\subsection{$\ell_p$ Loss}
\label{sec:Lp}

Here, we analyze $\ell_p$ loss functions for any $p \geq 1$: $\ell_p
\colon (h, x, y) \mapsto \abs*{h(x) - y}^p$. We show that this family
of loss functions benefits from $\sH$-consistency bounds with respect
to the squared loss assuming, when adopting the same symmetry and
boundedness assumptions as in the previous section.

\begin{restatable}{theorem}{Lp}
\label{thm:Lp}
Assume that the distribution is symmetric, and that the conditional
distribution and the hypothesis set $\sH$ are bounded by $B >
0$. Then, for all $h \in \sH$ and $p \geq 1$, the following
$\sH$-consistency bound holds:
\ifdim\columnwidth=\textwidth
{
\begin{align*}
\sE_{\ell_2}(h) - \sE^*_{\ell_2}(\sH) + \sM_{\ell_2}(\sH) \leq \Gamma \paren*{\sE_{\ell_p}(h)
    - \sE^*_{\ell_p}(\sH) + \sM_{\ell_p}(\sH)},
\end{align*}
}
\else{
\begin{align*}
& \sE_{\ell_2}(h) - \sE^*_{\ell_2}(\sH) + \sM_{\ell_2}(\sH)\\
  & \quad \leq \Gamma \paren*{\sE_{\ell_p}(h)
    - \sE^*_{\ell_p}(\sH) + \sM_{\ell_p}(\sH)},
\end{align*}
}
\fi
where $\Gamma(t) = \sup_{x \in \sX, y \in \sY} \curl*{|h(x) - y|
  + |\mu(x) - y|} \, t$ for $p = 1$, $\Gamma(t) = \frac{2}{(8B)^{p -
    2} p(p - 1)} \, t$ for $p \in (1, 2]$, and $\Gamma(t) =
t^{\frac{2}{p}}$ for $p \geq 2$.
\end{restatable}
The proof is included in Appendix~\ref{app:Lp}. Note that for $p = 1$,
$\Gamma$ can be further upper bounded as follows: $\Gamma (t) =
\sup_{x \in \sX}\sup_{y \in \sY} \curl*{|h(x) - y| + |\mu(x) - y|} \,
t \leq 4B \, t$ since the conditional distribution and the hypothesis
set $\sH$ are bounded by $B > 0$. This upper bound can also be
obtained by using general theorems in
Section~\ref{sec:general-theorems} with $\alpha \equiv 1$. However,
our generalized theorems, which apply to any positive function
$\alpha$, yield a finer bound for the $\ell_1$ loss. This further
shows that our generalized theorems are not only useful but can also
yield finer bounds.

The key term appearing in the bounds is the minimizability gap
$\sM_{\ell_p}(\sH) = \sE^*_{\ell_p}(\sH) -
\E_{x}\bracket*{\sC^*_{\ell_p}(\sH, x)}$, which is helpful for
comparing the bounds between $\ell_p$ losses for different $p \geq
1$. For example, for the $\ell_1$ and $\ell_2$ loss, by
Theorem~\ref{thm:general}, we have $\sM_{\ell_1}(\sH) =
\sE^*_{\ell_1}(\sH) - \E_{x}\bracket*{\E_y\bracket*{\abs*{\mu(x) -
      y}}}$ and $\sM_{\ell_2}(\sH) = \sE^*_{\ell_2}(\sH) -
\E_{x}\bracket*{\E_y\bracket*{\abs*{\mu(x) - y}^2}}$. Thus, in the
deterministic case, both $\E_y\bracket*{\abs*{\mu(x) - y}}$ and
$\E_y\bracket*{\abs*{\mu(x) - y}^2}$ vanish, and $\sM_{\ell_2} =
\sE^*_{\ell_2}(\sH) \geq \paren*{\sE^*_{\ell_1}(\sH)}^2 =
\paren*{\sM_{\ell_1}}^2$.

In particular, when $\sH$ is realizable, we have $\sM_{\ell_p}(\sH) =
\sM_{\ell_2}(\sH) = 0$. This yields the following result.
\begin{corollary}
\label{cor:Lp}
Assume that the distribution is symmetric, the conditional
distribution is bounded by $B > 0$, and the hypothesis set $\sH$ is
realizable and bounded by $B > 0$. Then, for all $h \in \sH$ and $p
\geq 1$, the following $\sH$-consistency bound holds:
\begin{align*}
\sE_{\ell_2}(h) - \sE^*_{\ell_2}(\sH) \leq \Gamma \paren*{\sE_{\ell_p}(h) - \sE^*_{\ell_p}(\sH)},
\end{align*}
where $\Gamma(t) = \sup_{x \in \sX, y \in \sY} \curl*{|h(x) - y| + |\mu(x) - y|} \, t$ for $p = 1$, $\Gamma(t) = \frac{2}{(8B)^{p - 2} p(p - 1)} \, t$ for $p \in (1, 2]$, and $\Gamma(t) = t^{\frac{2}{p}}$ for $p \geq 2$.
\end{corollary}
Corollary~\ref{cor:Lp} shows that when the estimation error of
$\ell_p$ is reduced to $\e$, the estimation error of the squared loss
$\paren*{\sE_{\ell_2}(h) - \sE^*_{\ell_2}(\sH)}$ is upper bounded by
$\e^{ \frac{2}{p} }$ for $p > 2$, and by $\e$ for $1 \leq p \leq 2$,
which is more favorable, modulo a multiplicative constant.

\subsection{Squared $\e$-insensitive Loss}
\label{sec:e-sen-sq}

The $\e$-insensitive loss and the squared $\e$-insensitive loss
functions are used in the support vector regression (SVR) algorithms
\citep{Vapnik2000}. The use of these loss functions results in sparser
solutions, characterized by fewer support vectors for the SVR
algorithms. Moreover, the selection of the parameter $\e$ determines a
trade-off between accuracy and sparsity: larger $\e$ values yield
increasingly sparser solutions. We first provide a positive result for
the squared $\e$-insensitive loss $\ell_{\rm{sq}-\e} \colon (h, x, y)
\mapsto \max \curl[\big]{\abs*{h(x) - y}^2 - \e^2, 0}$, by showing that it
admits an $\sH$-consistency bound with respect to $\ell_2$.

\begin{restatable}{theorem}{SensitiveSquared}
\label{thm:e-sen-sq}
Assume that the distribution is symmetric, and that the conditional
distribution and the hypothesis set $\sH$ are bounded by $B >
0$. Assume that $p_{\min}(\e) = \inf_{x \in \sX} \P(\mu(x) - y
\geq \e \mid x)$ is positive.  Then, for all $h \in \sH$, the
following $\sH$-consistency bound holds:
\ifdim\columnwidth=\textwidth
{
\begin{align*}
\sE_{\ell_2}(h) - \sE^*_{\ell_2}(\sH) + \sM_{\ell_2}(\sH) \leq \frac{\sE_{\ell_{\rm{sq}-\e}}(h) - \sE^*_{\ell_{\rm{sq}-\e}}(\sH) + \sM_{\ell_{\rm{sq}-\e}}(\sH)}{2 p_{\min}(\e)}.
\end{align*}
}
\else{
\begin{multline*}
\sE_{\ell_2}(h) - \sE^*_{\ell_2}(\sH) + \sM_{\ell_2}(\sH)\\
\quad \leq \frac{\sE_{\ell_{\rm{sq}-\e}}(h) - \sE^*_{\ell_{\rm{sq}-\e}}(\sH) + \sM_{\ell_{\rm{sq}-\e}}(\sH)}{2 p_{\min}(\e)}.
\end{multline*}
}
\fi
\end{restatable}
The proof is presented in Appendix~\ref{app:e-sq}. It requires the use
of Theorem~\ref{theorem:H-ConsBoundPsi} with $\alpha(h, x) =
\P(\mu(x) - y \geq \e \mid x)$. As in the case of the Huber
loss, the previous established general tools for $\sH$-consistency
bounds \citep{awasthi2022h,awasthi2022multi} do not apply here. 
Our generalization of previous tools proves essential for analyzing
$\sH$-consistency bounds in regression.
By Corollary~\ref{cor:general}, for realizable hypothesis sets, the
minimizability gap vanishes. Thus, by Theorem~\ref{thm:e-sen-sq}, we
obtain the following corollary.
\begin{corollary}
\label{cor:e-sen-sq}
Assume that the distribution is symmetric, the conditional
distribution is bounded by $B > 0$, and the hypothesis set $\sH$ is
realizable and bounded by $B > 0$.  Assume that $p_{\min}(\e) =
\inf_{x \in \sX} \P(\mu(x) - y \geq \e \mid x)$ is positive.
Then, for all $h \in \sH$, the following $\sH$-consistency bound
holds:
\begin{align*}
  \sE_{\ell_2}(h) - \sE^*_{\ell_2}(\sH)
  \leq \frac{\sE_{\ell_{\rm{sq}-\e}}(h) - \sE^*_{\ell_{\rm{sq}-\e}}(\sH)}
       {2 p_{\min}(\e)}.
\end{align*}
\end{corollary}
By taking the limit on both sides of the bound of
Corollary~\ref{cor:e-sen-sq}, we can infer the $\sH$-consistency of
the squared $\e$-insensitive loss under the assumption $p_{\min}(\e) >
0$.
Note that increasing $\e$ diminishes $p_{\min}(\e)$, making the bound
less favorable. Conversely, smaller $\e$ values enhance the linear
dependency bound but may hinder solution sparsity. Therefore,
selecting the optimal $\e$ involves balancing the trade-off between
linear dependency and sparsity.
When $p_{\min}(\e)$ approaches zero, the bound derived from
Corollary~\ref{cor:e-sen-sq} becomes less informative. However, as
demonstrated in the subsequent theorem, the squared $\e$-insensitive
loss fails to exhibit $\sH$-consistency with the squared loss if the
condition $p_{\min}(\e) > 0$ is not satisfied.

\begin{restatable}{theorem}{SensitiveSquaredN}
\label{thm:e-sen-sq-n}
Assume that the distribution is symmetric, the conditional
distribution is bounded by $B > 0$, and the hypothesis set $\sH$ is
realizable and bounded by $B > 0$. Then, the squared $\e$-insensitive
loss $\ell_{\rm{sq}-\e}$ is not $\sH$-consistent\ignore{ with respect to the
squared loss}.
\end{restatable}
The proof is given in Appendix~\ref{app:e-sq}. It consists of
considering a distribution that concentrates on an input $x$ with
$\P(Y = y \mid x) = \frac{1}{2} = \P(Y = 2\mu(x) - y \mid x)$, where
$-B \leq y < \mu(x) \leq B$ and $\mu(x) - y < \e$.  Then, we show that
both $\ov h \colon x \mapsto y + \e$ and $h^* \colon x \mapsto \mu(x)$
are best-in-class predictors of the squared $\e$-insensitive loss,
while the best-in-class-predictor of the squared loss is uniquely $h^*
\colon x \mapsto \mu(x)$.

\subsection{$\e$-Insensitive Loss}
\label{sec:e-sen}

In Appendix~\ref{app:e}, we present negative results, Theorem~\ref{thm:e-sen-more} and Theorem~\ref{thm:e-sen-less}, for the $\e$-insensitive loss
$\ell_{\e} \colon (h, x, y) \mapsto \max \curl*{\abs*{h(x) - y} - \e,
  0}$ used in the SVR algorithm, by showing that even under the
assumption $\inf_{x \in \sX} \P(\mu(x) - y \geq \e) > 0$ or $\inf_{x \in \sX} \P(0 \leq
\mu(x) - y \leq \e) > 0$, it is not $\sH$-consistent with respect to
the squared loss.

\subsection{Generalization bounds}
\label{sec:learning-bound}

We can use our $\sH$-consistency bounds to derive bounds on the
squared loss estimation error of a surrogate loss minimizer. For a
labeled sample $S = \paren*{(x_1, y_1), \ldots, (x_m, y_m)}$ drawn
i.i.d.\ according to $\sD$, let $\h h_S \in \sH$ be the empirical
minimizer of a regression loss function $L$ over $S$ and
$\Rad_{m}^{L}(\sH)$ the Rademacher complexity of the hypothesis set
$\curl*{(x, y) \mapsto \sfL(h(x), y) \colon h\in \sH}$. We denote by
$B_{L}$ an upper bound of the regression loss function $L$. Then, the
following generalization bound holds.
\begin{restatable}{theorem}{LearningBound}
\label{thm:learning-bound}
Assume that the distribution is symmetric, the conditional
distribution and the hypothesis set $\sH$ are bounded by $B >
0$. Then, for any $\delta > 0$, with probability at least $1-\delta$
over the draw of an i.i.d. sample $S$ of size $m$, the following
squared loss estimation bound holds for $\h h_S$:
\ifdim\columnwidth=\textwidth
{
\begin{align*}
\sE_{\ell_2}(\h h_S) - \sE^*_{\ell_2}(\sH) \leq \Gamma\paren*{\sM_{L}(\sH) + 4
    \Rad_m^{L}(\sH) + 2 B_{L} \sqrt{\tfrac{\log
        \frac{2}{\delta}}{2m}} } -  \sM_{\ell_2}(\sH).
\end{align*}
}
\else{
\begin{align*}
& \sE_{\ell_2}(\h h_S) - \sE^*_{\ell_2}(\sH)\\
& \quad \leq \Gamma\paren*{\sM_{L}(\sH) + 4
    \Rad_m^{L}(\sH) + 2 B_{L} \sqrt{\tfrac{\log
        \frac{2}{\delta}}{2m}} } -  \sM_{\ell_2}(\sH).
\end{align*}
}
\fi
where $\Gamma(t) = \sup_{x \in \sX}\sup_{y} \curl*{|\h h_S(x) - y| +
  |\mu(x) - y|}\, t$ for $L = \ell_1$, $\Gamma(t) = \frac{2}{(8B)^{p -
    2} p(p - 1)} \, t$ for $L = \ell_p$, $p \in (1, 2]$, $\Gamma(t) =
  t^{\frac{2}{p}}$ for $L = \ell_p$, $p \geq 2$, $\Gamma(t) =
  \frac{\max \curl*{\frac{2B}{\delta}, 2}}{p_{\min}(\delta)} \, t$ for $L = \ell_{\delta}$, and $\Gamma(t) =
  \frac{1}{2 p_{\min}(\e)} \, t$ for $L =
  \ell_{\rm{sq}-\e}$.
\end{restatable}
The proof is included in
Appendix~\ref{app:learning-bound}. Theorem~\ref{thm:learning-bound}
provides the first finite-sample guarantees for the squared loss
estimation error of the empirical minimizer of the Huber loss, squared
$\e$-insensitive loss, $\ell_1$ loss, and more generally $\ell_p$
loss. The proof leverages the $\sH$-consistency bounds of
Theorems~\ref{thm:huber}, \ref{thm:Lp}, \ref{thm:e-sen-sq}, along with
standard Rademacher complexity bounds \citep{mohri2018foundations}. Under the boundedness assumption, we have $\abs*{h(x) - y} \leq \abs*{h(x)} + \abs*{y} \leq 2B$. Thus, an upper bound $B_L$ for the regression loss function can be derived. For example, for the $\ell_p$ loss, we have $\abs*{h(x) - y}^p \leq (2B)^p$ and thus $B_L = (2B)^p$.

\section{Application to adversarial regression}
\label{sec:adversarial}

In this section, we show how the $\sH$-consistency guarantees we
presented in the previous section can be applied to the design of new
algorithms for adversarial regression. Deep neural networks are known
to be vulnerable to small adversarial perturbations around input data
\citep{KrizhevskySutskeverHinton2012,szegedy2013intriguing,
  SutskeverVinyalsLe2014,AwasthiMaoMohriZhong2023theoretically,awasthi2024dc}.

Despite extensive previous work aimed at improving the robustness of
neural networks, this often comes with a reduction in standard
(non-adversarial) accuracy, leading to a trade-off between adversarial
and standard generalization errors in both the classification
\citep{madry2017towards,tsipras2018robustness,zhang2019theoretically,
  raghunathan2019adversarial,min2021curious,javanmard2022precise,
  ma2022tradeoff,taheri2022asymptotic,dobriban2023provable} and
regression scenarios
\citep{javanmard2020precise,dan2020sharp,xing2021adversarially,
  hassani2022curse,liu2023non,ribeiro2023overparameterized,
  ribeiro2023regularization}.

In the context of adversarial classification,
\citet{zhang2019theoretically} proposed algorithms seeking a trade-off
between these two types of errors, by using the theory of
Bayes-consistent binary classification surrogate losses
functions. More recently, \citet{mao2023cross} introduced enhanced
algorithms by minimizing smooth adversarial comp-sum losses,
leveraging the $\sH$-consistency guarantee of comp-sum losses in
multi-class classification.

Building on the insights from these previous studies, we aim to
leverage our novel $\sH$-consistency theory tailored for standard
regression to introduce a family of new loss functions for adversarial
regression, termed as \emph{smooth adversarial regression
losses}. Minimizing these loss functions readily leads to new
algorithms for adversarial regression.

\subsection{Adversarial Squared Loss}
\label{sec:adv-squared}

In adversarial regression, the target adversarial generalization error
is measured by the worst squared loss under the bounded $\gamma$
perturbation of $x$. This is defined for any $(h, x, y) \in \sH \times
\sX \times \sY$ as follows:
\begin{equation*}
\wt \ell_2(h, x, y) = \sup_{x' \colon \norm*{x' - x} \leq \gamma} (h(x') - y)^2
\end{equation*}
where $\norm*{\, \cdot \,}$ denotes a norm on $\sX$, typically an $\ell_p$
norm for $p \geq 1$. We refer to $\wt \ell_2$ as the \emph{adversarial
squared loss}. By adding and subtracting the standard squared loss
$\ell_2$, for any $(h, x, y) \in
\sH \times \sX \times \sY$, we can write $\wt \ell_2$ as follows: 
$ \wt \ell_2(h, x, y)
   = (h(x) - y)^2
  + \sup_{x' \colon \norm*{x' - x} \leq \gamma} (h(x') - y)^2 - (h(x) - y)^2$.
Then, assuming that the conditional distribution and the hypothesis set $\sH$
are bounded by $B > 0$, we can write
\begin{align*}
& \sup_{x' \colon \norm*{x' - x} \leq \gamma} (h(x') - y)^2 - (h(x) - y)^2 \\ 
  & = \sup_{x' \colon \norm*{x' - x} \leq \gamma}
  \paren*{h(x') - h(x)}\paren*{h(x') + h(x) + y}\\
  & \leq  \sup_{x' \colon \norm*{x' - x} \leq \gamma} 3B \, \abs*{h(x') - h(x)}.
  \tag{$|h(x)| \leq B$, $|y| \leq B$}
\end{align*}
Thus, the $\wt \ell_2$ loss can be upper bounded as follows for all $(x, y)$:
\begin{equation}
\label{eq:adv-sq}
\wt \ell_2(h, x, y) \leq (h(x) - y)^2
+ \nu \sup_{x' \colon \norm*{x' - x} \leq \gamma} \abs*{h(x') - h(x)},
\end{equation}
where $\nu \geq 3B$ is a positive constant. 

\subsection{Smooth Adversarial Regression Losses}

Let $L$ be a standard regression loss function that admits an
$\sH$-consistency bound with respect to the squared loss:
\begin{equation*}
  \sE_{\ell_2}(h) - \sE^*_{\ell_2}(\sH)
  \leq \Gamma\paren*{\sE_{L}(h) - \sE^*_{L}(\sH)}.
\end{equation*}
To trade-off the adversarial and standard generalization errors in
regression, by using \eqref{eq:adv-sq}, we can upper bound the
difference between the adversarial generalization and the
best-in-class standard generalization error as follows:
\begin{align*}
  & \sE_{\wt \ell_2}(h) - \sE^*_{\ell_2}(\sH)\\
  & \leq \sE_{\ell_2}(h) - \sE^*_{\ell_2}(\sH)
  + \nu \sup_{x' \colon \norm*{x' - x} \leq \gamma} \abs*{h(x') - h(x)}\\
  & \leq \Gamma\paren*{\sE_{L}(h) - \sE^*_{L}(\sH)}
  + \nu \sup_{x' \colon \norm*{x' - x} \leq \gamma} \abs*{h(x') - h(x)}.
\end{align*}
Thus, by Corollaries~\ref{cor:huber},~\ref{cor:Lp}
and~\ref{cor:e-sen-sq}, we obtain the following guarantees with
respect to the adversarial squared loss. The proofs are presented in
Appendix~\ref{app:adv}.
\begin{restatable}{theorem}{AdvHuber}
\label{thm:adv-huber}
Assume that the distribution is symmetric, the conditional
distribution is bounded by $B > 0$, and the hypothesis set $\sH$ is
realizable and bounded by $B > 0$. Assume that $p_{\min}(\delta) = \inf_{x \in \sX} \P(0 \leq \mu(x) -
y \leq \delta \mid x)$ is positive. Then, for any $\nu
\geq 3B$ and all $h \in \sH$, the following bound holds:
\ifdim\columnwidth=\textwidth
{
\begin{align*}
\sE_{\wt \ell_2}(h) - \sE^*_{\ell_2}(\sH)
 \leq \frac{\max \curl*{\frac{2B}{\delta}, 2}}
  {p_{\min}(\delta)}
  \paren*{\sE_{\ell_{\delta}}(h) - \sE^*_{\ell_{\delta}}(\sH)} + \nu
  \sup_{x' \colon \norm*{x' - x} \leq \gamma} \abs*{h(x') - h(x)}
\end{align*}
}
\else{
\begin{align*}
  \sE_{\wt \ell_2}(h) - \sE^*_{\ell_2}(\sH)
  & \leq \frac{\max \curl*{\frac{2B}{\delta}, 2}}
  {p_{\min}(\delta)}
  \paren*{\sE_{\ell_{\delta}}(h) - \sE^*_{\ell_{\delta}}(\sH)}\\
  & \quad + \nu \mspace{-10mu}
  \sup_{x' \colon \norm*{x' - x} \leq \gamma} \abs*{h(x') - h(x)}.
\end{align*}
}
\fi
\end{restatable}
\begin{restatable}{theorem}{AdvLp}
\label{thm:adv-Lp}
Assume that the distribution is symmetric, the conditional
distribution is bounded by $B > 0$, and the hypothesis set $\sH$ is
realizable and bounded by $B > 0$. Then, for any $\nu \geq 3B$ and all
$h \in \sH$, the following bound holds:
\ifdim\columnwidth=\textwidth
{
\begin{align*}
\sE_{\wt \ell_2}(h) - \sE^*_{\ell_2}(\sH) \leq \Gamma \paren*{\sE_{\ell_p}(h) - \sE^*_{\ell_p}(\sH)}
  + \nu \sup_{x' \colon \norm*{x' - x} \leq \gamma} \abs*{h(x') - h(x)},
\end{align*}
}
\else{
\begin{align*}
  & \sE_{\wt \ell_2}(h) - \sE^*_{\ell_2}(\sH)\\
  & \quad \leq \Gamma \paren*{\sE_{\ell_p}(h) - \sE^*_{\ell_p}(\sH)}
  + \nu \mspace{-10mu} \sup_{x' \colon \norm*{x' - x} \leq \gamma} \abs*{h(x') - h(x)},
\end{align*}
}
\fi
where $\Gamma(t) = 
t^{\frac{2}{p}}$ if $p \geq 2$,
$\frac{2}{(8B)^{p - 2} p(p - 1)} \, t$ for $p \in (1, 2)$
and $4B \, t$, if $p = 1$.
\end{restatable}
\begin{restatable}{theorem}{AdvSenSq}
\label{thm:adv-sen-sq}
Assume that the distribution is symmetric, the conditional
distribution is bounded by $B > 0$, and the hypothesis set $\sH$ is
realizable and bounded by $B > 0$. Assume that $p_{\min}(\e) =
\inf_{x \in \sX} \P(\mu(x) - y \geq \e \mid x)$ is positive.
Then, for any $\nu \geq 3B$ and all
$h \in \sH$, the following bound holds:
\ifdim\columnwidth=\textwidth
{
\begin{align*}
\sE_{\wt \ell_2}(h) - \sE^*_{\ell_2}(\sH) \leq \frac{\sE_{\ell_{\rm{sq}-\e}}(h)
    - \sE^*_{\ell_{\rm{sq}-\e}}(\sH)}{2 p_{\min}(\e)}
  + \nu \sup_{x' \colon \norm*{x' - x} \leq \gamma} \abs*{h(x') - h(x)}.
\end{align*}
}
\else{
\begin{align*}
& \sE_{\wt \ell_2}(h) - \sE^*_{\ell_2}(\sH)\\
  & \quad \leq \frac{\sE_{\ell_{\rm{sq}-\e}}(h)
    - \sE^*_{\ell_{\rm{sq}-\e}}(\sH)}{2 p_{\min}(\e)}
  + \nu \mspace{-10mu} \sup_{x' \colon \norm*{x' - x} \leq \gamma} \abs*{h(x') - h(x)}.
\end{align*}
}
\fi
\end{restatable}
Theorems~\ref{thm:adv-huber}, \ref{thm:adv-Lp} and
\ref{thm:adv-sen-sq} suggest minimizing
\begin{equation}
\label{eq:smooth-adv}
L(h, x, y) + \tau \sup_{x' \colon \norm*{x' - x} \leq \gamma} \abs*{h(x') - h(x)} 
\end{equation}
where $L$ can be chosen as $\ell_{\delta}$, $\ell_p$ and
$\ell_{\rm{sq}-\e}$ and $\tau > 0$ is a parameter. For simplicity, we
use the parameter $\tau$ to approximate the effect of the functional
form $\Gamma$ in these bounds, as with the approach adopted in
\citep{zhang2019theoretically}.  Given that $L$ is a convex function
of $h$, the minimization of \eqref{eq:smooth-adv} can be achieved
equivalently and more efficiently by applying the standard Lagrange
method. This allows for the replacement of the $\ell_1$ norm with its
square, since the regularization term can be moved to a constraint,
where it can then be squared.  \ignore{
\begin{equation}
\label{eq:smooth-adv-2}
L(h, x, y) + \tau' \sup_{x' \colon \norm*{x' - x} \leq \gamma}
\abs*{h(x') - h(x)}^2,
\end{equation}
for some $\tau' > 0$.
}

We refer to \eqref{eq:smooth-adv} as \emph{smooth adversarial
regression loss functions}. They can be obtained by augmenting the
standard regression loss function such as the Huber loss, the $\ell_p$
loss and the $\e$-insensitive loss with a natural smoothness
term. Minimizing the regularized empirical smooth adversarial
regression loss functions leads to a new family of algorithms for
adversarial regression, \emph{smooth adversarial regression
algorithms}. In the next section, we report experimental results
illustrating the effectiveness of these new algorithms, in particular
in terms of the trade-off between the adversarial accuracy and
standard accuracy, as guaranteed by Theorems~\ref{thm:adv-huber},
\ref{thm:adv-Lp} and \ref{thm:adv-sen-sq}. We will show that these
algorithms outperform the direct minimization of the adversarial
squared loss.

It is important to note that the regularizer $\sup_{x' \colon \norm*{x' - x} \leq \gamma} \abs*{h(x') - h(x)}$ in the smooth adversarial regression loss closely relates to the local Lipschitz constant and the gradient norm, which are established methods in adversarially robust training \citep{hein2017formal,finlay2019scaleable,yang2020closer,gouk2021regularisation}. Furthermore, by building upon the derivation in Section~\ref{sec:adv-squared}, we can develop new surrogate losses for adversarial regression scenarios beyond the adversarial squared loss, such as the adversarial $\ell_1$ loss. In this context, the formulation of the smooth adversarial regression loss would replace the absolute value with the local Lipschitz constant of the target loss. To establish guarantees for these new surrogate losses, the $\sH$-consistency bounds shown in Section~\ref{sec:hcons} can be extended to other target losses in regression, such as the $\ell_1$ loss. An intriguing direction for future exploration is investigating how our surrogate losses relate to Moreau envelope theory (see, for example, \citep{zhou2022non}).

\begin{table}[t]
\vskip -0.3in
  \caption{Comparison of the performance of the {\sc Adv-sq} algorithm
    and our smooth adversarial regression algorithms for $L = \ell_2$
    and $L = \ell_{\delta}$ for $\ell_{\infty}$ adversarial training with
    perturbation size $\gamma \in \curl*{0.001, 0.005, 0.01} $ on the
    Diverse MAGIC wheat dataset.}
\label{tab:wheat}
\begin{center}
\begin{small}
\begin{sc}
\begin{tabular}{@{\hspace{0cm}}lccc@{\hspace{0cm}}}
\toprule
Method & Size & Clean & Robust \\
\midrule
Adv-sq   & \multirow{3}{*}{0.001} & \bm{$1.28 \pm 0.10$} & \bm{$1.32 \pm 0.11$} \\
Ours ($L = \ell_2$)  &            & \bm{$1.28 \pm 0.09$} & \bm{$1.32 \pm 0.09$} \\
Ours ($L = \ell_{\delta}$)  &     & $1.30 \pm 0.08$ & $1.34 \pm 0.09$ \\
\midrule
Adv-sq   & \multirow{3}{*}{0.005} & $1.30 \pm 0.09$ & $1.53 \pm 0.10$ \\
Ours ($L = \ell_2$)            &  & $1.26 \pm 0.09$ & $1.46 \pm 0.10$ \\
Ours ($L = \ell_{\delta}$)      & & \bm{$1.03 \pm 0.09$} & \bm{$1.12 \pm 0.10$} \\
\midrule
Adv-sq   & \multirow{3}{*}{0.01}  & $1.30 \pm 0.08$ & $1.78 \pm 0.11$ \\
Ours ($L = \ell_2$)  &            & $1.22 \pm 0.11$ & $1.62 \pm 0.14$ \\
Ours ($L = \ell_{\delta}$)      & & \bm{$0.97 \pm 0.02$} & \bm{$1.01 \pm 0.02$} \\
\bottomrule
\end{tabular}
\end{sc} 
\end{small}
\end{center}
\vskip -0.25in
\end{table}

\begin{table}[t]
\vskip -0.3in
\caption{Comparison of the performance of the {\sc Adv-sq} algorithm
  and our smooth adversarial regression algorithms for $L = \ell_2$
  and $L = \ell_{\delta}$ for $\ell_{\infty}$ adversarial training
  with perturbation size $\gamma \in \curl*{0.001, 0.005, 0.01} $ on
  the Diabetes dataset.}
\label{tab:diabetes}
\begin{center}
\begin{small}
\begin{sc}
\begin{tabular}{@{\hspace{0cm}}lccc@{\hspace{0cm}}}
\toprule
Method & Size & Clean & Robust \\
\midrule
Adv-sq   & \multirow{3}{*}{0.001} & $2.53 \pm 0.48$ & $2.57 \pm 0.49$ \\
Ours ($L = \ell_2$)  &            & \bm{$1.24 \pm 0.21$} & \bm{$1.26 \pm 0.21$} \\
Ours ($L = \ell_{\delta}$)  &     & $1.31 \pm 0.15$ & $1.32 \pm 0.15$ \\
\midrule
Adv-sq   & \multirow{3}{*}{0.005} & $1.12 \pm 0.12$ & $1.18 \pm 0.13$ \\
Ours ($L = \ell_2$)            &  & $0.80 \pm 0.04$ & $0.82 \pm 0.04$ \\
Ours ($L = \ell_{\delta}$)      & & \bm{$0.78 \pm 0.06$} & \bm{$0.79 \pm 0.06$} \\
\midrule
Adv-sq   & \multirow{3}{*}{0.01}  & $0.83 \pm 0.05$ & $0.87 \pm 0.05$ \\
Ours ($L = \ell_2$)  &            & \bm{$0.74 \pm 0.05$} & \bm{$0.76 \pm 0.05$} \\
Ours ($L = \ell_{\delta}$)      & & $0.81 \pm 0.05$ & $0.82 \pm 0.05$ \\
\bottomrule
\end{tabular}
\end{sc}
\end{small}
\end{center}
\vskip -0.25in
\end{table}

\section{Experiments}
\label{sec:experiments}

In this section, we demonstrate empirically the effectiveness of the
smooth adversarial regression algorithms introduced in the previous
section.

\textbf{Experimental settings.}  We studied two real-world datasets:
the Diabetes dataset \citep{efron2004least} and the Diverse MAGIC
wheat dataset \citep{scott2021limited}, and adopted the same exact
settings for feature engineering as
\citep[Example 3 and Example 5 in Appendix~D]{ribeiro2023regularization}. For the sake of a fair comparison,
we used a linear hypothesis set.
We considered an $\ell_{\infty}$
perturbation with perturbation size $\gamma \in \curl*{0.001, 0.005,
  0.01}$ for adversarial training. For our smooth adversarial
regression losses \eqref{eq:smooth-adv}, we chose $L = \ell_2$, the
squared loss, and $L = \ell_{\delta}$ with $\delta = 0.2$, the Huber
loss, setting $\tau = 1$ as the default. Other choices for the
regression loss functions and the value of $\tau$ may yield better
performance, which can typically be selected by cross-validation in
practice. Both our smooth adversarial regression losses and the
adversarial squared loss were optimized using the CVXPY library
\citep{diamond2016cvxpy}.

\textbf{Evaluation.} We report the standard error, measured by the
squared loss (or MSE) on the test data, and the robust error, measured
by the adversarial squared loss with $\ell_{\infty}$ perturbation and
the corresponding perturbation size used for training. We averaged
both errors over five runs and report the standard deviation for
both our smooth adversarial regression losses and the adversarial
squared loss.

\textbf{Results.}  Tables~\ref{tab:wheat} and~\ref{tab:diabetes}
present the experimental results of our adversarial regression
algorithms with both the squared ($L = \ell_2$) and Huber ($L =
\ell_{\delta}$) losses. The results suggest that these algorithms
consistently surpass the adversarial squared loss in clean and robust
error metrics across various settings. In particular, on the Diabetes dataset with a perturbation size of $\gamma = 0.01$, our method ($L = \ell_2$) outperforms the adversarial squared loss by more than $0.5\%$ in both robust error and clean error. Similarly, on the
    Diverse MAGIC wheat dataset with a perturbation size of $\gamma = 0.01$, our method ($L = \ell_{\delta}$) surpasses the adversarial squared loss by more than $0.3\%$ in terms of robust and clean errors.

Remarkably, the surrogate loss
using the Huber loss occasionally outperforms the squared loss
variant. This highlights the importance of using surrogate losses,
even within the adversarial training framework, to enhance
performance.

\section{Conclusion}

We presented the first study of $\sH$-consistency bounds for
regression.This involved generalizing existing tools that were
previously used to prove $\sH$-consistency bounds.  Leveraging our
generalized tools, we proved a series of novel $\sH$-consistency
bounds for surrogate losses of the squared loss. Our $\sH$-consistency
guarantees can be beneficial in designing new algorithms for
adversarial regression. This study can be useful for the later studies
of other surrogate losses for other target losses in regression.

\bibliography{reghcb}

\newpage
\appendix

\renewcommand{\contentsname}{Contents of Appendix}
\tableofcontents
\addtocontents{toc}{\protect\setcounter{tocdepth}{3}} 
\clearpage

\section{Proofs of general \texorpdfstring{$\sH$}{H}-consistency theorems}
\label{app:general}
\subsection{Proof of Theorem~\ref{theorem:H-ConsBoundPsi}}
\HConsBoundPsi*
\begin{proof}
For any $h\in \sH$, we can write
\begin{align*}
  & \Psi\paren*{\sE_{L_2}(h) - \sE_{L_2,\sH}^* + \sM_{L_2,\sH}}\\
  & = \Psi\paren*{\E_{X} \bracket*{\Delta\sC_{L_2,\sH}(h, x)}}\\
  & \leq \E_{X} \bracket*{\Psi\paren*{\Delta\sC_{L_2,\sH}(h, x)}}
  \tag{Jensen's ineq.} \\
  & = \E_{X} \bracket*{\Psi\paren*{\Delta\sC_{L_2,\sH}(h, x) 1_{\Delta\sC_{L_2,\sH}(h, x)>\e} + \Delta\sC_{L_2,\sH}(h, x) 1_{\Delta\sC_{L_2,\sH}(h, x)\leq\e}}}\\ 
  & \leq \E_{X} \bracket*{\Psi\paren*{\Delta\sC_{L_2,\sH}(h, x) 1_{\Delta\sC_{L_2,\sH}(h, x) > \e}} + \Psi\paren*{\Delta\sC_{L_2,\sH}(h, x) 1_{\Delta\sC_{L_2,\sH}(h, x)\leq\e}}}
  \tag{$\Psi(0) \geq 0$}\\ 
  & \leq \E_{X} \bracket*{\alpha(h, x) \Delta\sC_{L_1,\sH}(h, x)}
  + \sup_{t\in[0, \e]}\Psi(t)
  \tag{assumption}\\
  & \leq \bracket[\Big]{\sup_{x \in \sX} \alpha(h, x)} \E_x \bracket*{\Delta\sC_{L_1,\sH}(h, x)}
  + \sup_{t\in[0, \e]}\Psi(t)
  \tag{H\"older’s ineq.}\\
  & = \bracket[\Big]{\sup_{x \in \sX} \alpha(h, x)} \paren*{\sE_{L_1}(h) - \sE_{L_1,\sH}^* + \sM_{L_1,\sH}}
  + \max \curl*{\Psi(0), \Psi(\e)},
  \tag{convexity of $\Psi$}
\end{align*}
which completes the proof.
\end{proof}

\subsection{Proof of Theorem~\ref{theorem:H-ConsBoundGamma}}
\HConsBoundGamma*
\begin{proof}
For any $h\in \sH$, we can write
\begin{align*}
  & \sE_{L_2}(h) - \sE_{L_2, \sH}^* + \sM_{L_2, \sH}\\
  & = \E_X  \bracket*{\sC_{L_2}(h, x) - \sC^*_{L_2, \sH}(x)}\\
  & = \E_X  \bracket*{\Delta\sC_{L_2, \sH}(h, x)}\\
  & = \E_X  \bracket*{\Delta\sC_{L_2, \sH}(h, x) 1_{\Delta\sC_{L_2, \sH}(h, x) > \e} + \Delta\sC_{L_2, \sH}(h, x)1_{\Delta\sC_{L_2, \sH}(h, x) \leq \e}}\\ 
  & \leq \E_X \bracket*{\Gamma\paren*{\alpha(h, x) \Delta\sC_{L_1, \sH}(h, x)}} + \e
  \tag{assumption}\\ 
  & \leq \Gamma\paren*{\E_X \bracket*{\alpha(h, x) \Delta\sC_{L_1, \sH}(h, x)}} + \e
  \tag{Jensen's ineq.}\\ 
  & \leq \Gamma \paren*{ \bracket[\Big]{\sup_{x \in \sX} \alpha(h, x)} \E_X \bracket*{\Delta\sC_{L_1, \sH}(h, x)}} + \e
  \tag{H\"older's ineq.}\\
  & = \Gamma\paren*{\bracket[\Big]{\sup_{x \in \sX} \alpha(h, x)} \paren*{\sE_{L_1}(h) - \sE_{L_1, \sH}^*+\sM_{L_1, \sH}}} + \e.
\end{align*}
When $\Gamma(x) = x^{\frac{1}{q}}$ for some $q \geq 1$ with conjugate
number $p$, starting from the fourth inequality above, we can write
\begin{align*}
  \sE_{L_2}(h) - \sE_{L_2, \sH}^* + \sM_{L_2, \sH}
  & \leq \E_X \bracket*{\alpha^{\frac{1}{q}}(h, x) \Delta\sC_{L_2,\sH}^{\frac{1}{q}}(h, x)} + \e\\ 
  & \leq \E_X \bracket*{\alpha^{\frac{p}{q}}(h, x)}^{\frac{1}{p}} \E_X \bracket*{\Delta\sC_{L_1,\sH}(h, x)}^{\frac{1}{q}} + \e
  \tag{H\"older's ineq}\\
  & = \E_X \bracket*{\alpha^{\frac{p}{q}}(h, x)}^{\frac{1}{p}} \E_X \bracket*{\sE_{L_1}(h) - \sE_{L_1,\sH}^* + \sM_{L_1,\sH}}^{\frac{1}{q}} + \e.
\end{align*}
This completes the proof.
\end{proof}

\subsection{Proof of Theorem~\ref{thm:squared-bound}}
\SquaredLossBound*
\begin{proof}
By definition, 
\begin{align*}
\sC^*_{\ell_2}(\sH, x) &= \inf_{h \in \sH} \E \bracket*{ \paren*{h(x) - y}^2 \mid x}\\
& = \inf_{h \in \sH} \bracket*{ \paren*{h(x) - \E \bracket*{y \mid x} }^2 + \E[y^2 \mid x] - \paren*{\E[y \mid x]}^2 }\\
& = \E[y^2 \mid x] - \paren*{\E[y \mid x]}^2\\
\Delta \sC_{\ell_2, \sH}(h, x) &= \E \bracket*{\paren*{h(x) - y}^2 \mid x} - \inf_{h \in \sH} \E \bracket*{ \paren*{h(x) - y}^2 \mid x}\\
& = \paren*{h(x) - \E \bracket*{y \mid x} }^2 + \E[y^2 \mid x] - \paren*{\E[y \mid x]}^2 - \paren*{\E[y^2 \mid x] - \paren*{\E[y \mid x]}^2}\\
& = \paren*{h(x) - \E \bracket*{y \mid x} }^2.
\end{align*}
This completes the proof.
\end{proof}

\subsection{Proof of Theorem~\ref{thm:general}}
\General*
\begin{proof}
By the symmetry of the distribution, we can write
\begin{align*}
\E_y \bracket*{\psi(h(x) - y) \mid x} 
&= \frac{\E_y \bracket*{\psi(h(x) - y) \mid x} + \E_y \bracket*{\psi(h(x) - 2
\mu(x) + y) \mid x}}{2}\\
&=\frac{\E_y \bracket*{\psi(h(x) - y) \mid x} + \E_y \bracket*{\psi(-h(x) + 2
\mu(x) - y) \mid x}}{2} \tag{$\psi$ is symmetric}\\
& = \frac{\E_y \bracket*{\psi(h(x) - y) + \psi(-h(x) + 2 \mu(x) - y) \mid x}}{2}\\
&\geq \E_y \bracket*{ \psi \paren*{\mu(x) - y} } \tag{Jensen's inequality}
\end{align*}
where the equality is achieved when $h(x) = \mu(x) \in \sH$. This completes the proof.
\end{proof}

\section{Proofs of \texorpdfstring{$\sH$}{H}-consistency bounds for common surrogate losses}
\label{app:commom}

\subsection{\texorpdfstring{$\sH$}{H}-consistency of \texorpdfstring{$\ell_{\delta}$}{L} with respect to \texorpdfstring{$\ell_2$}{L2}}
\label{app:huber}
Define the function $g$ as $g \colon t \mapsto \frac12 t^2 1_{\abs*{t} \leq \delta} + \paren*{\delta \abs*{t} - \frac12 \delta^2} 1_{\abs*{t} > \delta}$. Consider the function $F$ defined over $\bracket*{-B, B}^2$ by $F(x, y) = \frac{g(x + y) + g(x - y)}{2} - g(y)$. We prove a useful lemma as follows.
\begin{lemma}
\label{lemma:ineq_huber}
For any $x, y \in [-B, B]$ and $ \abs*{y} \leq \delta$, the following inequality holds: \[
F(x, y) \geq \min \curl*{\frac{\delta}{2B}, \frac14} x^2.\]
\end{lemma}
\begin{proof}
Given the definition of $g$ and the symmetry of $F$ with respect to $y = 0$, we can assume, without loss of generality, that $y \geq 0$. Next, we will analyze case by case.

\noindent \textbf{Case I:} $\abs*{x + y} \leq \delta$, $\abs*{x - y} \leq \delta $, $0 \leq y \leq \delta$. In this case, we have 
\begin{align*}
F(x, y) = \frac{\frac12 (x + y)^2 + \frac12 (x - y)^2}{2} - \frac12 y^2 = \frac{1}{2} x^2 \geq \min \curl*{\frac{\delta}{2B}, \frac14} x^2.
\end{align*}

\noindent \textbf{Case II:} $\abs*{x + y} \leq \delta$, $\abs*{x - y} > \delta $, $ 0 \leq y \leq \delta$. In this case, we must have $-y - \delta \leq x < y -\delta$ and $\delta \geq y \geq \max \curl*{- x - \delta, x + \delta} \geq x + \delta$. Thus,
\begin{align*}
F(x, y)
& = \frac{\frac12 (x + y)^2 + \delta \abs*{x - y} - \frac12 \delta^2}{2} - \frac12 y^2\\
& = \frac{\frac12 (x + y)^2 + \delta (y - x) - \frac12 \delta^2}{2} - \frac12 y^2 \tag{$x - y < 0$}\\
& = \frac{-\frac12 y^2 + (x + \delta)y + \frac12 x^2 - \delta x - \frac12 \delta^2}{2}\\
& \geq \frac{-\frac12 \delta^2 + (x + \delta)\delta + \frac12 x^2 - \delta x - \frac12 \delta^2}{2}
\tag{the minimum of the quadratic function is attained when $y = \delta$}\\
& = \frac{x^2}{4}\\
& \geq \min \curl*{\frac{\delta}{2B}, \frac14} x^2.
\end{align*}

\noindent \textbf{Case III:} $\abs*{x + y} > \delta$, $\abs*{x - y} \leq \delta $, $ 0 \leq y \leq \delta$. In this case, we must have $-y + \delta \leq x \leq y + \delta$ and $\delta \geq y \geq \max \curl*{- x + \delta, x - \delta} \geq -x + \delta$. Thus,
\begin{align*}
F(x, y)
& = \frac{\delta \abs*{x + y} - \frac12 \delta^2 + \frac12 (x - y)^2}{2} - \frac12 y^2\\
& = \frac{\delta \paren*{x + y} - \frac12 \delta^2 + \frac12 (x - y)^2}{2} - \frac12 y^2 \tag{$x + y > 0$}\\
& = \frac{-\frac12 y^2 + (-x + \delta)y + \frac12 x^2 + \delta x - \frac12 \delta^2}{2}\\
& \geq \frac{-\frac12 \delta^2 + (-x + \delta)\delta + \frac12 x^2 + \delta x - \frac12 \delta^2}{2}
\tag{the minimum of the quadratic function is attained when $y = \delta$}\\
& = \frac{x^2}{4}\\
& \geq \min \curl*{\frac{\delta}{2B}, \frac14} x^2.
\end{align*}

\noindent \textbf{Case IV:} $x + y > \delta$, $\abs*{x - y} > \delta $, $ 0 \leq y \leq \delta$. In this case, we must have $x > y + \delta \geq \delta$ and $ 0 \leq y < \min\curl*{x - \delta, \delta}$. Thus, we have
\begin{align*}
F(x, y)
& = \frac{\delta \abs*{x + y} - \frac12 \delta^2 + \delta \abs*{x - y} - \frac12 \delta^2}{2} - \frac12 y^2\\
& = \frac{\delta \paren*{x + y} - \frac12 \delta^2 + \delta \paren*{x - y} - \frac12 \delta^2}{2} - \frac12 y^2 \tag{$x + y > 0$ and $x - y > 0$}\\
& = \frac{- y^2 + 2\delta x - \delta^2}{2}.
\end{align*}
Then, if $\delta < x \leq 2 \delta$ and $\min\curl*{x - \delta, \delta} = x - \delta$,
\begin{align*}
F(x, y)
& = \frac{- y^2 + 2\delta x - \delta^2}{2}\\
& \geq \frac{- (x - \delta)^2 + 2\delta x - \delta^2}{2}
\tag{the minimum of the quadratic function is attained when $y = x - \delta$}\\
& = \frac{- x^2 + 4\delta x - 2\delta^2}{2}\\
& \geq \frac{x^2}{4} \tag{$\delta < x \leq 2 \delta$}\\
& \geq \min \curl*{\frac{\delta}{2B}, \frac14} x^2.
\end{align*}
If $ 2 \delta < x \leq B $ and $\min\curl*{x - \delta, \delta} = \delta$,
\begin{align*}
F(x, y)
& = \frac{- y^2 + 2\delta x - \delta^2}{2}\\
& \geq \frac{- \delta^2 + 2\delta x - \delta^2}{2}
\tag{the minimum of the quadratic function is attained when $y = \delta$}\\
& = \delta x - \delta^2\\
& \geq \frac{\delta}{2B} x^2 \tag{$2 \delta < x \leq B$}\\
& \geq \min \curl*{\frac{\delta}{2B}, \frac14} x^2.
\end{align*}

\noindent \textbf{Case V:} $x + y < -\delta$, $\abs*{x - y} > \delta $, $ 0 \leq y \leq \delta$. In this case, we must have $x < -y -\delta \leq -\delta$, and $ 0 \leq y < \min\curl*{-x - \delta, \delta}$. Thus, we have
\begin{align*}
F(x, y)
& = \frac{\delta \abs*{x + y} - \frac12 \delta^2 + \delta \abs*{x - y} - \frac12 \delta^2}{2} - \frac12 y^2\\
& = \frac{-\delta \paren*{x + y} - \frac12 \delta^2  -\delta \paren*{x - y} - \frac12 \delta^2}{2} - \frac12 y^2 \tag{$x + y < 0$ and $x - y < 0$}\\
& = \frac{-y^2 - 2\delta x - \delta^2}{2}.
\end{align*}
Then, if $-2\delta \leq x < -\delta$ and $\min\curl*{-x - \delta, \delta} = -x - \delta$,
\begin{align*}
F(x, y) 
& = \frac{-y^2 - 2\delta x - \delta^2}{2}\\
& \geq \frac{-(-x - \delta)^2 - 2\delta x - \delta^2}{2}
\tag{the minimum of the quadratic function is attained when $y = -x - \delta$}\\
& = \frac{-x^2 - 4\delta x - 2\delta^2}{2}\\
& \geq \frac{x^2}{4} \tag{$-2\delta \leq x < -\delta$}\\
& \geq \min \curl*{\frac{\delta}{2B}, \frac14} x^2.
\end{align*}
If $-B \leq x < -2\delta$ and $\min\curl*{-x - \delta, \delta} = \delta$,
 \begin{align*}
F(x, y) 
& = \frac{-y^2 - 2\delta x - \delta^2}{2}\\
& \geq \frac{-\delta^2 - 2\delta x - \delta^2}{2}
\tag{the minimum of the quadratic function is attained when $y = \delta$}\\
& = -\delta x - \delta^2\\
& \geq \frac{\delta}{2B} x^2 \tag{$-B \leq x < -2\delta$}\\
& \geq \min \curl*{\frac{\delta}{2B}, \frac14} x^2.
\end{align*}
In summary, we complete the proof.
\end{proof}

\Huber*
\begin{proof}
By Theorem~\ref{thm:general}, we can write $\forall h \in \sH, x \in \sX$,
\begin{align*}
  & \Delta \sC_{\ell_{\delta}}(h, x)\\
  & = \E_y \bracket*{ \frac12 \paren*{h(x) - y}^2 1_{\abs*{h(x) - y} \leq \delta } + \paren*{\delta \abs*{h(x) - y} - \frac12 \delta^2} 1_{\abs*{h(x) - y} > \delta } \mid x}\\
  & \qquad - \E_y \bracket*{ \frac12 \paren*{\mu(x) - y}^2 1_{\abs*{\mu(x) - y} \leq \delta } + \paren*{\delta \abs*{\mu(x) - y} - \frac12 \delta^2} 1_{\abs*{\mu(x) - y} > \delta } \mid x}\\
  & =  \E_y \bracket*{ \frac{g(h(x) - \mu(x) + \mu(x) - y) + g(h(x) - \mu(x) - \paren*{\mu(x) - y})}{2} - g(\mu(x) - y) \mid x} \tag{distribution is symmetric with respect to $\mu(x)$}\\
  & =  \E_y \bracket*{ F(h(x) - \mu(x), \mu(x) - y) \mid x}\\
  & \geq 2 \P(0 \leq \mu(x) - y \leq \delta \mid x) \E_y \bracket*{ F(h(x) - \mu(x), \mu(x) - y) \mid 0 \leq \mu(x) - y \leq \delta}\\
  & \geq \P(0 \leq \mu(x) - y \leq \delta \mid x)\min \curl*{\frac{\delta}{2B}, \frac12} \paren*{h(x) - \mu(x)}^2 \tag{$\abs*{h(x) - \mu(x)} \leq 2B, \abs*{\mu(x) - y} \leq 2B$}\\
  & = \P(0 \leq \mu(x) - y \leq \delta \mid x) \min \curl*{\frac{\delta}{2B}, \frac12} \Delta \sC_{\ell_2}(h, x).
\end{align*}
By Theorems~\ref{theorem:H-ConsBoundPsi} or \ref{theorem:H-ConsBoundGamma} with $\alpha(h, x) = \frac{1}{\P(0 \leq \mu(x) - y \leq \delta \mid x)}$,  we have
\begin{equation*}
\sE_{\ell_2}(r) - \sE^*_{\ell_2}(\sR) + \sM_{\ell_2}(\sR) \leq \frac{\max \curl*{\frac{2B}{\delta}, 2}}{p_{\min}(\delta)} \paren*{\sE_{\ell_{\delta}}(r) - \sE^*_{\ell_{\delta}}(\sR) + \sM_{\ell_{\delta}}(\sR)}.
\end{equation*}
\end{proof}

\HuberN*
\begin{proof}
Consider a distribution that concentrates on an input $x$. Choose $y, \mu(x), \delta \in \Rset$ such that $-B \leq y < \mu(x) \leq B$ and $\mu(x) - y > \delta$. Consider the conditional distribution as $\P(Y = y \mid x) = \frac12 = \P(Y = 2\mu(x) - y \mid x)$. Thus, the distribution is symmetric with respect to $y = \mu(x)$. For such a distribution, the best-in-class predictor for the squared loss is
$h^*(x) = \mu(x)$.
However, for the Huber loss, we have 
\begin{align*}
& \sC_{\ell_{\delta}}(h, x)\\
&= \E_y \bracket*{ \frac12 \paren*{h(x) - y}^2 1_{\abs*{h(x) - y} \leq \delta } + \paren*{\delta \abs*{h(x) - y} - \frac12 \delta^2} 1_{\abs*{h(x) - y} > \delta } \mid x}\\
&= \frac12 \paren*{\frac12 \paren*{h(x) - y}^2 1_{\abs*{h(x) - y} \leq \delta } + \paren*{\delta \abs*{h(x) - y} - \frac12 \delta^2} 1_{\abs*{h(x) - y} > \delta }}\\
&\quad + \frac12 \paren*{\frac12 \paren*{h(x) - 2\mu(x) + y}^2 1_{\abs*{h(x) - 2\mu(x) + y} \leq \delta } + \paren*{\delta \abs*{h(x) - 2\mu(x) + y} - \frac12 \delta^2} 1_{\abs*{h(x) - 2\mu(x) + y} > \delta }}.
\end{align*}
Thus, plugging $\ov h \colon x \mapsto y + \delta$ and $h^* \colon x \mapsto \mu(x)$, we obtain that
\begin{align*}
\sC_{\ell_{\delta}}(\ov h, x) 
& = \frac12 \paren*{\frac12 \delta^2} + \frac12 \paren*{\delta \abs*{2y + \delta - 2\mu(x)} - \frac12 \delta^2} \tag{$\ov h(x) - y = \delta$ and $\ov h(x) - 2\mu(x) + y < -\delta$}\\
& = \delta\paren*{\mu(x) - \frac12 \delta - y}\\
\sC_{\ell_{\delta}}(h^*, x) 
& = \frac12 \paren*{\delta \abs*{\mu(x) - y} - \frac12 \delta^2 + \delta \abs*{-\mu(x) + y} - \frac12 \delta^2} \tag{$h^*(x) - y > \delta$ and $h^*(x) - 2\mu(x) + y < -\delta$}\\
& = \delta\paren*{\mu(x) - \frac12 \delta - y}.
\end{align*}
Therefore, $\sC_{\ell_{\delta}}(\ov h, x)  = \sC_{\ell_{\delta}}(h^*, x) $, and both $\ov h$ and $h^*$ are the best-in-class predictors for the Huber loss. This implies that the Huber loss is not $\sH$-consistent with respect to the squared loss.
\end{proof}

\subsection{\texorpdfstring{$\sH$}{H}-consistency of \texorpdfstring{$\ell_p$}{Lp} with respect to \texorpdfstring{$\ell_2$}{L2}}
\label{app:Lp}
\Lp*
\begin{proof}
We will analyze case by case.

\textbf{Case I: $p \geq 2$.} 
By Theorem~\ref{thm:general}, we can write
\begin{align*}
\forall h \in \sH, x \in \sX, \quad 
  & \Delta \sC_{\ell_p}(h, x)\\
  & = \E_y\bracket*{\abs*{h(x) - y}^p -  \abs*{\mu(x) - y}^p \mid x}\\
  & = \E_y\bracket*{\frac{\abs*{h(x) - y}^p + \abs*{h(x) - 2\mu(x) + y}^p}{2} -  \abs*{\mu(x) - y}^p \mid x} \tag{distribution is symmetric with respect to $\mu(x)$}\\
  & = \E_y\bracket*{\frac{\abs*{h(x) - \mu(x) + \mu(x) - y}^p + \abs*{h(x) - \mu(x) - (\mu(x) - y)}^p}{2} -  \abs*{\mu(x) - y}^p \mid x}\\
  & \geq \abs*{h(x) - \mu(x)}^p \tag{by Clarkson's inequality \citep{clarkson1936uniformly}}\\
  & =  \paren*{\paren*{h(x) - \mu(x)}^2}^{\frac{p}{2}}\\
  & = \paren*{\Delta \sC_{\ell_2}(h, x)}^{\frac{p}{2}}.
\end{align*}
By Theorem~\ref{theorem:H-ConsBoundPsi}, we have
\begin{equation*}
\sE_{\ell_2}(r) - \sE^*_{\ell_2}(\sR) + \sM_{\ell_2}(\sR) \leq \paren*{\sE_{\ell_p}(r) - \sE^*_{\ell_p}(\sR) + \sM_{\ell_p}(\sR)}^{\frac{2}{p}}.
\end{equation*}

\textbf{Case II: $1 < p \leq 2$.} 
In this case, the Clarkson's inequality cannot be used directly. We first prove a useful lemma as follows.
\begin{lemma}
\label{lemma:ineq3}
For any $x, y \in [-B, B]$ and $1 < p \leq 2$, the following inequality holds: \[
\frac{|x + y|^{p} + |x - y|^p}{2} - |y|^{p} \geq \frac{(2B)^{p - 2} p(p - 1)}{2} x^2.\]
\end{lemma}
\begin{proof}
For any $y \in [-B, B]$, consider the function $f_y \colon x \mapsto \frac{|x + y|^{p} + |x - y|^p}{2} - |y|^{p} - \frac{(2B)^{p - 2} p(p - 1)}{2} x^2$.
We compute the first derivative and second derivative of $f_y$ as follows:
\begin{align*}
f'_y(x) &= \dfrac{\frac{p |x + y|^p}{x + y} + \frac{p | x - y |^p}{x - y}}{2} - (2B)^{p - 2}p(p - 1) x\\
f''_y(x) &= \dfrac{\frac{p(p - 1)}{ |x + y |^{2 - p}} + \frac{p(p - 1)}{|x - y|^{2 - p}}}{2} - (2B)^{p - 2}p(p - 1).
\end{align*}
Thus, using the fact that $1 < p \leq 2$ and $|x + y| \leq 2B$, $|x - y| \leq 2B$, we have
\begin{equation*}
\forall x \in [-B, B], \quad f''_y(x) \geq \dfrac{\frac{p(p - 1)}{(2B)^{2 - p}} + \frac{p(p - 1)}{(2B)^{2 - p}}}{2} - (2B)^{p - 2}p(p - 1) = 0.
\end{equation*}
Therefore, $f_y(x)$ is convex. Since $f'_y(0) = 0$, $x = 0$ achieves the minimum:
\begin{equation*}
\forall x,y \in [-B, B], \quad
f_y(x) \geq f_y(0) = 0.
\end{equation*}
This completes the proof.
\end{proof}
By Theorem~\ref{thm:general}, we can write
\begin{align*}
\forall h \in \sH, x \in \sX, \quad 
  & \Delta \sC_{\ell_p}(h, x)\\
  & = \E_y\bracket*{\abs*{h(x) - y}^p -  \abs*{\mu(x) - y}^p \mid x}\\
  & = \E_y\bracket*{\frac{\abs*{h(x) - y}^p + \abs*{h(x) - 2\mu(x) + y}^p}{2} -  \abs*{\mu(x) - y}^p \mid x} \tag{distribution is symmetric with respect to $\mu(x)$}\\
  & = \E_y\bracket*{\frac{\abs*{h(x) - \mu(x) + \mu(x) - y}^p + \abs*{h(x) - \mu(x) - (\mu(x) - y)}^p}{2} -  \abs*{\mu(x) - y}^p \mid x}\\
  & \geq \frac{(8B)^{p - 2} p(p - 1)}{2} \paren*{h(x) - \mu(x)}^2 \tag{by Lemma~\ref{lemma:ineq3} and $\abs*{h(x) - \mu(x)}\leq 4B, \abs*{\mu(x) - y}\leq 4B$}\\
  & = \frac{(8B)^{p - 2} p(p - 1)}{2} \Delta \sC_{\ell_2}(h, x).
\end{align*}
By Theorem~\ref{theorem:H-ConsBoundPsi}, we have
\begin{equation*}
\sE_{\ell_2}(r) - \sE^*_{\ell_2}(\sR) + \sM_{\ell_2}(\sR) \leq \frac{2}{(8B)^{p - 2} p(p - 1)} \paren*{\sE_{\ell_p}(r) - \sE^*_{\ell_p}(\sR) + \sM_{\ell_p}(\sR)}.
\end{equation*}

\textbf{Case III: $p = 1$.}
By Theorem~\ref{thm:general}, we can write
\begin{align*}
\forall h \in \sH, x \in \sX, \quad \Delta \sC_{\ell_2}(h, x)
  & = \E_y \bracket*{(h(x) - y)^2 -  (\mu(x) - y)^2 \mid x}\\
  & = \E_y \bracket*{\paren*{|h(x) - y| + |\mu(x) - y|} \paren*{|h(x) - y| - |\mu(x) - y|} \mid x}\\
  & \leq \sup_{y \in \sY} \curl*{|h(x) - y| + |\mu(x) - y|}  \E_y \bracket*{|h(x) - y| - |\mu(x) - y| \mid x}\\
  & = \sup_{y \in \sY} \curl*{|h(x) - y| + |\mu(x) - y|} \, \Delta \sC_{\ell_1}(h, x).
\end{align*}
By Theorems~\ref{theorem:H-ConsBoundPsi} or \ref{theorem:H-ConsBoundGamma} with $\alpha(h, x) = \sup_{y \in \sY} \curl*{|h(x) - y| + |\mu(x) - y|}$, we have
\begin{equation*}
\sE_{\ell_2}(r) - \sE^*_{\ell_2}(\sR) + \sM_{\ell_2}(\sR) \leq \sup_{x \in \sX}\sup_{y \in \sY} \curl*{|h(x) - y| + |\mu(x) - y|} \paren*{\sE_{\ell_1}(r) - \sE^*_{\ell_1}(\sR) + \sM_{\ell_1}(\sR)}.
\end{equation*}
\end{proof}

\subsection{\texorpdfstring{$\sH$}{H}-consistency of \texorpdfstring{$\ell_{\rm{sq}-\e}$}{LepsilonSquared} with respect to \texorpdfstring{$\ell_2$}{L2}} 
\label{app:e-sq}
Define the function $g$ as $g \colon t \mapsto \max\curl*{t^2 - \e^2, 0}$. Consider the function $F$ defined over $\Rset^2$ by $F(x, y) = \frac{g(x + y) + g(x - y)}{2} - g(y)$. We first prove a useful lemma as follows.
\begin{lemma}
\label{lemma:ineq_sensitive-sq}
For any $x \in \Rset$ and $\abs*{y} \geq \e$, the following inequality holds: \[
F(x, y) \geq x^2. \]
\end{lemma}
\begin{proof}
Given the definition of $g$ and the symmetry of $F$ with respect to $y = 0$, we can assume, without loss of generality, that $y \geq 0$. Next, we will analyze case by case.

\noindent \textbf{Case I:} $\abs*{x + y} > \e$, $\abs*{x - y} > \e $, $y \geq \e$. In this case, we have 
\begin{align*}
F(x, y) = \frac{(x + y)^2 -\e^2 + (x - y)^2 - \e^2}{2} - y^2 + \e^2 = x^2.
\end{align*}

\noindent \textbf{Case II:} $\abs*{x + y} > \e$, $\abs*{x - y} \leq \e$, $ y \geq \e$. In this case, we must have $y - \e \leq x \leq y + \e$ and $x + \e \geq y \geq \max \curl*{x - \e, \e} \geq x - \e$. Thus,
\begin{align*}
F(x, y)
& = \frac{(x + y)^2 - \e^2 + 0}{2} - y^2 + \e^2\\
& = \frac{- y^2 + 2xy + x^2 + \e^2}{2}\\
& \geq \frac{- (x + \e)^2 + 2x(x + \e) + x^2 + \e^2}{2}
\tag{the minimum of the quadratic function is attained when $y = x + \e$}\\
& = x^2.    
\end{align*}

\noindent \textbf{Case III:} $\abs*{x + y} \leq \e$, $\abs*{x - y} > \e$, $ y \geq \e$. In this case, we must have $ -y - \e \leq x \leq -y + \e$ and $-x + \e \geq y \geq \max \curl*{-x - \e, \e} \geq -x - \e$. Thus,
\begin{align*}
F(x, y)
& = \frac{0 + (x - y)^2 - \e^2}{2} - y^2 + \e^2\\
& = \frac{- y^2 - 2xy + x^2 + \e^2}{2}\\
& \geq \frac{- (-x + \e)^2 - 2x(-x + \e) + x^2 + \e^2}{2}
\tag{the minimum of the quadratic function is attained when $y = -x + \e$}\\
& = x^2.    
\end{align*}

\noindent \textbf{Case IV:} $\abs*{x + y} \leq \e$, $\abs*{x - y} \leq \e $, $y \geq \e$. In this case, we must have $x = 0$ and $y = \e$. Thus,
\begin{align*}
F(x, y) = \frac{0 + 0}{2} - 0 = 0 = x^2.
\end{align*}
In summary, we complete the proof.
\end{proof}

\SensitiveSquared*
\begin{proof}
By Theorem~\ref{thm:general}, we can write $\forall h \in \sH, x \in \sX$,
\begin{align*}
  & \Delta \sC_{\ell_{\rm{sq}-\e}}(h, x)\\
  & = \E_y \bracket*{ \max\curl*{ \paren*{h(x) - y}^2, \e^2} \mid x} - \E_y \bracket*{ \max\curl*{\paren*{\mu(x) - y}^2, \e^2} \mid x}\\
  & =  \E_y \bracket*{ \frac{g(h(x) - \mu(x) + \mu(x) - y) + g(h(x) - \mu(x) - \paren*{\mu(x) - y})}{2} - g(\mu(x) - y) \mid x} \tag{distribution is symmetric with respect to $\mu(x)$}\\
  & =  \E_y \bracket*{ F(h(x) - \mu(x), \mu(x) - y) \mid x}\\
  & \geq 2 \P(\mu(x) - y \geq \e \mid x) \E_y \bracket*{ F(h(x) - \mu(x), \mu(x) - y) \mid \mu(x) - y \geq \e}\\
  & \geq 2\P(\mu(x) - y \geq \e \mid x) \paren*{h(x) - \mu(x)}^2 \tag{by Lemma~\ref{lemma:ineq_sensitive-sq}}\\
  & = 2\P(\mu(x) - y \geq \e \mid x) \Delta \sC_{\ell_2}(h, x). 
\end{align*}
By Theorems~\ref{theorem:H-ConsBoundPsi} or \ref{theorem:H-ConsBoundGamma} with $\alpha(h, x) = \frac{1}{2 \P(\mu(x) - y \geq \e \mid x)}$, we have
\begin{equation*}
\sE_{\ell_2}(r) - \sE^*_{\ell_2}(\sR) + \sM_{\ell_2}(\sR) \leq \frac{\sE_{\ell_{\rm{sq}-\e}}(r) - \sE^*_{\ell_{\rm{sq}-\e}}(\sR) + \sM_{\ell_{\rm{sq}-\e}}(\sR)}{2p_{\min}(\e)}.
\end{equation*}
\end{proof}

\SensitiveSquaredN*
\begin{proof}
Consider a distribution that concentrates on an input $x$. Choose $y, \mu(x), \e \in \Rset$ such that $-B \leq y < \mu(x) \leq B$ and $\mu(x) - y < \e$. Consider the conditional distribution as $\P(Y = y \mid x) = \frac12 = \P(Y = 2\mu(x) - y \mid x)$. Thus, the distribution is symmetric with respect to $y = \mu(x)$. For such a distribution, the best-in-class predictor for the squared loss is
$h^*(x) = \mu(x)$.
However, for the $\e$-insensitive loss, we have 
\begin{align*}
& \sC_{\ell_{\rm{sq}-\e}}(h, x)\\
&= \E_y \bracket*{ \max \curl*{ \paren*{h(x) - y}^2 - \e^2, 0} \mid x}\\
&= \frac12 \max \curl*{ \paren*{h(x) - y}^2 - \e^2, 0} + \frac12 \max \curl*{ \paren*{h(x) - 2\mu(x) + y}^2 - \e^2, 0}.
\end{align*}
Thus, plugging $\ov h \colon x \mapsto y + \e$ and $h^* \colon x \mapsto \mu(x)$, we obtain that
\begin{align*}
\sC_{\ell_{\rm{sq}-\e}}(\ov h, x) 
& = \frac12 \paren*{0} + \frac12  \paren*{0} \tag{$\ov h(x) - y = \e$ and $\e > \ov h(x) - 2\mu(x) + y > -\e$}\\
& = 0\\
\sC_{\ell_{\rm{sq}-\e}}(h^*, x) 
& = \frac12 \paren*{0} + \frac12  \paren*{0} \tag{$ 0 < h^*(x) - y < \e$ and $0 > h^*(x) - 2\mu(x) + y > -\e$}\\
& = 0.
\end{align*}
Therefore, $\sC_{\ell_{\rm{sq}-\e}}(\ov h, x)  = \sC_{\ell_{\rm{sq}-\e}}(h^*, x) $, and both $\ov h$ and $h^*$ are the best-in-class predictors for the $\e$-insensitive loss. This implies that the $\e$-insensitive loss is not $\sH$-consistent with respect to the squared loss.
\end{proof}

\subsection{\texorpdfstring{$\sH$}{H}-consistency of \texorpdfstring{$\ell_{\e}$}{Lepsilon} with respect to \texorpdfstring{$\ell_2$}{L2}} 
\label{app:e}

Here, we present negative results for the $\e$-insensitive loss
$\ell_{\e} \colon (h, x, y) \mapsto \max \curl*{\abs*{h(x) - y} - \e,
  0}$ used in the SVR algorithm, by showing that even under the
assumption $\inf_{x \in \sX} \P(\mu(x) - y \geq \e) > 0$ or $\inf_{x \in \sX} \P(0 \leq
\mu(x) - y \leq \e) > 0$, it is not $\sH$-consistent with respect to
the squared loss.\ignore{The proofs of Theorems~\ref{thm:e-sen-more} and \ref{thm:e-sen-less}
are included in Appendix~\ref{app:e}.} In the proof, we consider
distributions that concentrate on an input $x$, leading to both $\ov h
\colon x \mapsto y + \e$ and $h^* \colon x \mapsto \mu(x)$ being the
best-in-class predictors for the $\e$-insensitive loss.

\begin{restatable}{theorem}{SensitiveMore}
\label{thm:e-sen-more}
Assume that the distribution is symmetric and satisfies $\inf_{x \in
  \sX} \P(\mu(x) - y \geq \e \mid x) > 0$. Assume that the
conditional distribution is bounded by $B > 0$, and the hypothesis set
$\sH$ is realizable and bounded by $B > 0$. Then, the $\e$-insensitive
loss $\ell_{\e}$ is not $\sH$-consistent with respect to the squared
loss.
\end{restatable}
\begin{proof}
Consider a distribution that concentrates on an input $x$. Choose $y, \mu(x), \e \in \Rset$ such that $-B \leq y < \mu(x) \leq B$ and $\mu(x) - y > \e$. Consider the conditional distribution as $\P(Y = y \mid x) = \frac12 = \P(Y = 2\mu(x) - y \mid x)$. Thus, the distribution is symmetric with respect to $y = \mu(x)$. For such a distribution, the best-in-class predictor for the squared loss is
$h^*(x) = \mu(x)$.
However, for the $\e$-insensitive loss, we have 
\begin{align*}
& \sC_{\ell_{\rm{sq}-\e}}(h, x)\\
&= \E_y \bracket*{ \max \curl*{ \abs*{h(x) - y} - \e, 0} \mid x}\\
&= \frac12 \max \curl*{ \abs*{h(x) - y} - \e, 0} + \frac12 \max \curl*{ \abs*{h(x) - 2\mu(x) + y} - \e, 0}.
\end{align*}
Thus, plugging $\ov h \colon x \mapsto y + \e$ and $h^* \colon x \mapsto \mu(x)$, we obtain that
\begin{align*}
\sC_{\ell_{\rm{sq}-\e}}(\ov h, x) 
& = \frac12 \paren*{0} + \frac12  \paren*{2\mu(x) - 2y - 2\e} \tag{$\ov h(x) - y = \e$ and $\ov h(x) - 2\mu(x) + y < -\e$}\\
& = \mu(x) - y - \e.\\
\sC_{\ell_{\rm{sq}-\e}}(h^*, x) 
& = \frac12 \paren*{\mu(x) - y - \e} + \frac12  \paren*{\mu(x) - y - \e} \tag{$h^*(x) - y > \e$ and $h^*(x) - 2\mu(x) + y < -\e$}\\
& = \mu(x) - y - \e.
\end{align*}
Therefore, $\sC_{\ell_{\rm{sq}-\e}}(\ov h, x)  = \sC_{\ell_{\rm{sq}-\e}}(h^*, x) $, and both $\ov h$ and $h^*$ are the best-in-class predictors for the $\e$-insensitive loss. This implies that the $\e$-insensitive loss is not $\sH$-consistent with respect to the squared loss.
\end{proof}

\begin{restatable}{theorem}{SensitiveLess}
\label{thm:e-sen-less}
Assume that the distribution is symmetric and satisfies $p_{\min}(\e)
= \inf_{x \in \sX} \P(0 \leq \mu(x) - y \leq \e \mid x) >
0$. Assume further that the conditional distribution is bounded by $B
> 0$, and the hypothesis set $\sH$ is realizable and bounded by $B >
0$. Then, the $\e$-insensitive loss $\ell_{\e}$ is not
$\sH$-consistent with respect to the squared loss.
\end{restatable}

\begin{proof}
Consider a distribution that concentrates on an input $x$. Choose $y, \mu(x), \e \in \Rset$ such that $-B \leq y < \mu(x) \leq B$ and $\mu(x) - y < \e$. Consider the conditional distribution as $\P(Y = y \mid x) = \frac12 = \P(Y = 2\mu(x) - y \mid x)$. Thus, the distribution is symmetric with respect to $y = \mu(x)$. For such a distribution, the best-in-class predictor for the squared loss is
$h^*(x) = \mu(x)$.
However, for the $\e$-insensitive loss, we have 
\begin{align*}
& \sC_{\ell_{\rm{sq}-\e}}(h, x)\\
&= \E_y \bracket*{ \max \curl*{ \abs*{h(x) - y} - \e, 0} \mid x}\\
&= \frac12 \max \curl*{ \abs*{h(x) - y} - \e, 0} + \frac12 \max \curl*{ \abs*{h(x) - 2\mu(x) + y} - \e, 0}.
\end{align*}
Thus, plugging $\ov h \colon x \mapsto y + \e$ and $h^* \colon x \mapsto \mu(x)$, we obtain that
\begin{align*}
\sC_{\ell_{\rm{sq}-\e}}(\ov h, x) 
& = \frac12 \paren*{0} + \frac12  \paren*{0} \tag{$\ov h(x) - y = \e$ and $\e > \ov h(x) - 2\mu(x) + y > -\e$}\\
& = 0.\\
\sC_{\ell_{\rm{sq}-\e}}(h^*, x) 
& = \frac12 \paren*{0} + \frac12  \paren*{0} \tag{$0 < h^*(x) - y < \e$ and $0 > h^*(x) - 2\mu(x) + y > -\e$}\\
& = 0.
\end{align*}
Therefore, $\sC_{\ell_{\rm{sq}-\e}}(\ov h, x)  = \sC_{\ell_{\rm{sq}-\e}}(h^*, x) $, and both $\ov h$ and $h^*$ are the best-in-class predictors for the $\e$-insensitive loss. This implies that the $\e$-insensitive loss is not $\sH$-consistent with respect to the squared loss.
\end{proof}

\section{Proofs of generalization bound}
\label{app:learning-bound}
\LearningBound*
\begin{proof}
  By using the standard Rademacher complexity bounds \citep{mohri2018foundations}, for any $\delta>0$,
  with probability at least $1 - \delta$, the following holds for all $h \in \sH$:
\[
\abs*{\sE_{L}(h) - \h\sE_{L,S}(h)}
\leq 2 \Rad_m^{L}(\sH) +
B_{L} \sqrt{\tfrac{\log (2/\delta)}{2m}}.
\]
Fix $\e > 0$. By the definition of the infimum, there exists $h^* \in
\sH$ such that $\sE_{L}(h^*) \leq
\sE_{L}^*(\sH) + \e$. By definition of
$\h h_S$, we have
\begin{align*}
  & \sE_{L}(\h h_S) - \sE_{L}^*(\sH)\\
  & = \sE_{L}(\h h_S) - \h\sE_{L,S}(\h h_S) + \h\sE_{L,S}(\h h_S) - \sE_{L}^*(\sH)\\
  & \leq \sE_{L}(\h h_S) - \h\sE_{L,S}(\h h_S) + \h\sE_{L,S}(h^*) - \sE_{L}^*(\sH)\\
  & \leq \sE_{L}(\h h_S) - \h\sE_{L,S}(\h h_S) + \h\sE_{L,S}(h^*) - \sE_{L}^*(h^*) + \e\\
  & \leq
  2 \bracket*{2 \Rad_m^{L}(\sH) +
B_{L} \sqrt{\tfrac{\log (2/\delta)}{2m}}} + \e.
\end{align*}
Since the inequality holds for all $\e > 0$, it implies:
\[
\sE_{L}(\h h_S) - \sE_{L}^*(\sH)
\leq 
4 \Rad_m^{L}(\sH) +
2 B_{L} \sqrt{\tfrac{\log (2/\delta)}{2m}}.
\]
Plugging in this inequality in the bound
of Theorems~\ref{thm:huber}, \ref{thm:Lp}, \ref{thm:e-sen-sq} completes the proof.
\end{proof}

\section{Proofs of adversarial regression}
\label{app:adv}
\subsection{Proof of Theorem~\ref{thm:adv-huber}}
\AdvHuber*
\begin{proof}
By \eqref{eq:adv-sq}, we have
\begin{align*}
\sE_{\wt \ell_2}(h) - \sE^*_{\ell_2}(\sH) 
& \leq \sE_{\ell_2}(h) - \sE^*_{\ell_2}(\sH) + \nu \sup_{x' \colon \norm*{x' - x} \leq \gamma} \abs*{h(x') - h(x)}\\   
& \leq  \frac{\max \curl*{\frac{2B}{\delta}, 2}}{p_{\min}(\delta)} \paren*{\sE_{\ell_{\delta}}(h) - \sE^*_{\ell_{\delta}}(\sH)} + \nu \sup_{x' \colon \norm*{x' - x} \leq \gamma} \abs*{h(x') - h(x)} \tag{Corollary~\ref{cor:huber}}.
\end{align*}
This completes the proof.
\end{proof}

\subsection{Proof of Theorem~\ref{thm:adv-Lp}}
\AdvLp*
\begin{proof}
By \eqref{eq:adv-sq}, we have
\begin{align*}
\sE_{\wt \ell_2}(h) - \sE^*_{\ell_2}(\sH) 
& \leq \sE_{\ell_2}(h) - \sE^*_{\ell_2}(\sH) + \nu \sup_{x' \colon \norm*{x' - x} \leq \gamma} \abs*{h(x') - h(x)}\\   
& \leq  \Gamma \paren*{\sE_{\ell_p}(h) - \sE^*_{\ell_p}(\sH)} + \nu \sup_{x' \colon \norm*{x' - x} \leq \gamma} \abs*{h(x') - h(x)} \tag{Corollary~\ref{cor:Lp}}.
\end{align*}
where $\Gamma(t) = \begin{cases}
t^{\frac{2}{p}} & p > 2\\
\frac{2}{(8B)^{p - 2} p(p - 1)} \, t & p \in (1, 2]\\
4B \, t & p = 1.\\
\end{cases}$. This completes the proof.
\end{proof}

\subsection{Proof of Theorem~\ref{thm:adv-sen-sq}}
\AdvSenSq*
\begin{proof}
By \eqref{eq:adv-sq}, we have
\begin{align*}
\sE_{\wt \ell_2}(h) - \sE^*_{\ell_2}(\sH) 
& \leq \sE_{\ell_2}(h) - \sE^*_{\ell_2}(\sH) + \nu \sup_{x' \colon \norm*{x' - x} \leq \gamma} \abs*{h(x') - h(x)}\\   
& \leq  \frac{\sE_{\ell_{\rm{sq}-\e}}(h) - \sE^*_{\ell_{\rm{sq}-\e}}(\sH)}{2 p_{\min}(\e)} + \nu \sup_{x' \colon \norm*{x' - x} \leq \gamma} \abs*{h(x') - h(x)} \tag{Corollary~\ref{cor:e-sen-sq}}.
\end{align*}
This completes the proof.
\end{proof}

\end{document}